\documentclass[12pt,a4paper]{article}

\usepackage[utf8]{inputenc}
\usepackage[T1]{fontenc}
\usepackage{amsmath,amssymb,amsthm}
\usepackage{mathtools}
\usepackage{graphicx}
\usepackage{booktabs}
\usepackage{tabularx}
\usepackage{multirow}
\usepackage{float}
\usepackage[margin=2.5cm]{geometry}
\usepackage{setspace}
\usepackage{natbib}
\usepackage{hyperref}
\usepackage{cleveref}
\usepackage{caption}
\usepackage{subcaption}
\usepackage[ruled,vlined]{algorithm2e}

\newtheorem{proposition}{Proposition}
\newtheorem{remark}{Remark}

\title{\textbf{Variational Inference for Bayesian MIDAS Regression}}
\author{Luigi Simeone\thanks{Independent Researcher. Contact: luigi.simeone@tech-management.net.}}
\date{February 2026}

\begin{document}

\maketitle

\begin{abstract}
\noindent We develop a Coordinate Ascent Variational Inference (CAVI) algorithm for Bayesian Mixed Data Sampling (MIDAS) regression with linear weight parameterizations. The model separates impact coefficients from weighting function parameters through a normalization constraint, creating a bilinear structure that renders generic Hamiltonian Monte Carlo samplers unreliable while preserving conditional conjugacy exploitable by CAVI. Each variational update admits a closed-form solution: Gaussian for regression coefficients and weight parameters, Inverse-Gamma for the error variance. The algorithm propagates uncertainty across blocks through second moments, distinguishing it from naive plug-in approximations. In a Monte Carlo study spanning 21 data-generating configurations with up to 50 predictors, CAVI produces posterior means nearly identical to a block Gibbs sampler benchmark while achieving speedups of $107\times$ to $1{,}772\times$ (Table~\ref{tab:timing}). Generic automatic differentiation VI (ADVI), by contrast, produces bias 7--14 times larger while being orders of magnitude slower, confirming the value of model-specific derivations. Weight function parameters maintain excellent calibration (coverage above 92\%) across all configurations. Impact coefficient credible intervals exhibit the underdispersion characteristic of mean-field approximations, with coverage declining from 89\% to 55\% as the number of predictors grows---a documented trade-off between speed and interval calibration that structured variational methods can address. An empirical application to realized volatility forecasting on S\&P~500 daily returns confirms that CAVI and Gibbs sampling yield virtually identical point forecasts, with CAVI completing each monthly estimation in under 10 milliseconds.


\end{abstract}

\newpage

\section{Introduction}

Mixed Data Sampling (MIDAS) regression, introduced by \citet{ghysels2006}, provides a parsimonious framework for incorporating high-frequency data into low-frequency prediction equations. The approach has become standard in macroeconomic nowcasting and financial volatility forecasting, where daily or intraday observations carry information about monthly or quarterly aggregates. Bayesian treatments of MIDAS models offer natural uncertainty quantification and regularization through prior distributions---features particularly valuable when the number of high-frequency predictors is moderate to large \citep{mogliani2021,kohns2025}.

Posterior inference in Bayesian MIDAS models relies almost exclusively on Markov Chain Monte Carlo (MCMC) methods. The canonical framework of \citet{chan2025} employs precision-based sampling \citep{chanjeliazkov2009} for time-varying parameter MIDAS with linear weight parameterizations. \citet{pettenuzzo2016} use a four-block Gibbs sampler extending Kim-Shephard-Chib methods, while \citet{mogliani2021} develop adaptive MCMC for hierarchical MIDAS with 134 predictors. These methods are asymptotically exact but computationally intensive: each posterior evaluation requires thousands of sampling iterations, and mixing efficiency degrades as the parameter space grows.

Variational inference (VI) offers a fundamentally different approach, recasting posterior approximation as an optimization problem \citep{blei2017}. Rather than generating dependent samples from the posterior, VI seeks the closest member of a tractable distribution family by maximizing the evidence lower bound (ELBO). The resulting deterministic algorithm converges in a fraction of the time required by MCMC, producing both point estimates and approximate uncertainty quantification. When conditional conjugacy is unavailable, automatic differentiation VI \citep{kucukelbir2017} provides a generic black-box implementation using stochastic gradients; when the model structure permits, coordinate ascent methods yield deterministic closed-form updates. In the mixed-frequency econometrics literature, \citet{gefang2020} demonstrated that variational Bayes methods for large mixed-frequency VARs produce accurate results while being orders of magnitude faster than MCMC. Their subsequent work \citep{gefang2023} extended VI to VARs with horseshoe, LASSO, and stochastic volatility specifications. \citet{chanyu2022} contributed global Gaussian variational approximations for log-volatilities in large Bayesian VARs, and \citet{loaiza2022} developed structured variational approximations for time-varying parameter VAR models that substantially improve upon mean-field methods.

Despite this progress in mixed-frequency VARs, \textbf{variational inference has never been applied to MIDAS regression models}. This gap is methodologically significant. MIDAS regressions are structurally distinct from mixed-frequency VARs: the separation between weighting function parameters~$\boldsymbol{\theta}_j$ (controlling \emph{how} high-frequency observations are aggregated) and impact coefficients~$\beta_j$ (controlling \emph{how much} each predictor matters) generates a bilinear product $\beta_j \cdot \tilde{x}_t^{(j)}(\boldsymbol{\theta}_j)$ that is absent in the VAR setting. This bilinear structure has three consequences for posterior computation that motivate the present paper.

\emph{First}, generic Hamiltonian Monte Carlo (HMC) samplers---including the widely used No-U-Turn Sampler (NUTS; \citealp{hoffman2014})---fail on this model. The bilinear product creates a posterior geometry analogous to Neal's funnel \citep{neal2003}: when $|\beta_j|$ is large, the posterior is tightly concentrated in~$\boldsymbol{\theta}_j$; when $|\beta_j|$ is near zero, $\boldsymbol{\theta}_j$ is nearly unidentified. This curvature variation across the parameter space causes NUTS trajectories to diverge, producing biased estimates and unreliable diagnostics. Our implementation of NUTS for the MIDAS model consistently exhibited numerical divergences and non-convergent chains, confirming that off-the-shelf HMC is not a viable option.

\emph{Second}, well-designed block Gibbs samplers---which condition on one parameter block at a time, collapsing the bilinear term to a linear regression---do produce valid posterior samples. However, the alternating conditional structure creates autocorrelation between successive draws that worsens with the number of predictors. In our Monte Carlo experiments, the minimum effective sample size (ESS) of the Gibbs sampler drops from 4,103 at $J=1$ to 539 at $J=10$ and further to 134 at $J=25$, indicating that an increasing fraction of computational effort produces redundant samples.

\emph{Third}, the very conditional structure that enables Gibbs sampling---conditional on~$\boldsymbol{\theta}_j$, the model is linear in $(\alpha, \boldsymbol{\beta})$; conditional on $(\alpha, \boldsymbol{\beta})$, it is linear in each~$\boldsymbol{\theta}_j$---also enables Coordinate Ascent Variational Inference (CAVI) with closed-form updates. This is the central observation of the paper. The linear weight parameterization introduced by \citet{chan2025}---using Almon polynomials, Fourier series, or B-splines as basis functions---maintains conditional conjugacy across all parameter blocks. We exploit this to derive a CAVI algorithm where every update step has an analytic solution: Gaussian for regression coefficients and weight parameters, Inverse-Gamma for the error variance.

\paragraph{Contribution.} This paper makes three contributions. First, we develop the first variational inference algorithm for MIDAS regression, deriving closed-form CAVI updates that handle the bilinear structure through a mean-field factorization with variance-corrected moments. The updates propagate uncertainty across blocks via second moments---specifically, the variance of~$\beta_j$ enters the precision matrix of~$\boldsymbol{\eta}_j$, and the variance of~$\boldsymbol{\eta}_j$ enters the design matrix for~$\boldsymbol{\xi}$---producing more calibrated estimates than naive plug-in approaches. The covariance correction in the weight parameter update, arising from the joint treatment of $(\alpha, \boldsymbol{\beta})$, is a technical contribution specific to the MIDAS bilinear setting.

Second, we provide a comprehensive Monte Carlo evaluation across 21 data-generating configurations spanning $J \in \{1, 3, 5, 10, 25, 50\}$ predictors, $T \in \{50, 100, 200, 400\}$ observations, varying signal-to-noise ratios, weight profile shapes, basis function types, and frequency ratios. CAVI posterior means are nearly identical to Gibbs: bias differs by less than 0.03 across configurations, including $J=25$ where 50 Gibbs replications confirm close agreement. Computational speedups range from $107\times$ ($J=25$) to $1{,}772\times$ ($J=1$) in the baseline tier (Table~\ref{tab:scalability}), reaching down to $136\times$ under stress-test conditions with $K=65$ lags (Table~\ref{tab:timing}), and grow with the number of predictors---precisely where MCMC becomes most expensive. Generic ADVI, by contrast, produces bias 7--14 times larger while being 2,000--100,000 times slower, confirming that model-specific derivations are essential. Weight function parameters ($\boldsymbol{\eta}$) maintain coverage above 92\% in all configurations. Impact coefficient ($\beta$) credible intervals exhibit the underdispersion characteristic of mean-field VI, with coverage declining from 89\% ($J=1$) to 55\% ($J=50$)---a pattern we explain through the interaction between dimension-dependent cross-block correlations and the mean-field factorization.

Third, we demonstrate the method on a canonical empirical application: forecasting realized volatility of the S\&P~500 using daily squared returns in a MIDAS-RV specification \citep{ghysels2006}. Over 187 out-of-sample months, CAVI and Gibbs produce virtually identical forecasts (relative MSE difference below 0.01\%), with CAVI completing each estimation in under 10 milliseconds versus 1--2.5 seconds for Gibbs.

\paragraph{Positioning.} This paper occupies a specific niche in the literature. Relative to \citet{chan2025}, we adopt their model specification---the linear parameterization and null-space reparameterization of the normalization constraint---but replace their precision-based MCMC with variational inference. Relative to \citet{gefang2020}, we translate variational methods from mixed-frequency VARs to MIDAS regression, handling the bilinear $\beta \cdot \boldsymbol{\theta}$ structure and normalization constraint that are specific to MIDAS and absent in VARs. Relative to generic ADVI \citep{kucukelbir2017}, our model-specific updates are both orders of magnitude faster and substantially more accurate (Section~4.5), because they exploit conditional conjugacy that stochastic gradient methods cannot leverage.

The paper proceeds as follows. Section~2 specifies the Bayesian MIDAS model. Section~3 derives the CAVI algorithm. Section~4 reports the Monte Carlo study. Section~5 presents the empirical application. Section~6 concludes.

\section{Bayesian MIDAS Model}

\subsection{Observation equation}

Consider a low-frequency dependent variable $y_t$, $t = 1, \ldots, T$, and $J$ high-frequency predictors. For predictor~$j$, denote the vector of $K_j$ high-frequency observations associated with low-frequency period~$t$ as $\mathbf{x}_t^{(j)} = (x_{t,0}^{(j)}, \ldots, x_{t,K_j-1}^{(j)})^\top \in \mathbb{R}^{K_j}$. The MIDAS regression model is
\begin{equation}\label{eq:obs}
y_t = \alpha + \sum_{j=1}^{J} \beta_j \, \tilde{x}_t^{(j)}(\boldsymbol{\theta}_j) + \varepsilon_t, \qquad \varepsilon_t \overset{\text{iid}}{\sim} \mathcal{N}(0, \sigma^2), \qquad t = 1, \ldots, T,
\end{equation}
where $\alpha$ is the intercept, $\beta_j$ is the impact coefficient for predictor~$j$, and the MIDAS aggregation is
\begin{equation}\label{eq:agg}
\tilde{x}_t^{(j)}(\boldsymbol{\theta}_j) = \sum_{k=0}^{K_j - 1} w_j(k; \boldsymbol{\theta}_j) \, x_{t,k}^{(j)}.
\end{equation}
The separation of~$\beta_j$ (how much predictor~$j$ matters) from~$\boldsymbol{\theta}_j$ (how its high-frequency lags are weighted) is a defining feature of MIDAS.

\subsection{Linear weight parameterization}

The original MIDAS specification of \citet{ghysels2006} employs nonlinear weighting functions---exponential Almon or beta polynomials---that create computational difficulties: the posterior conditionals of~$\boldsymbol{\theta}_j$ do not belong to standard families, ruling out Gibbs sampling and requiring particle filters \citep{schumacher2014}. \citet{foroni2015} demonstrated that unrestricted lag polynomials estimated by OLS can perform comparably for small frequency ratios, motivating the shift toward linear weight specifications.

Following \citet{chan2025}, we adopt a \textbf{linear parameterization}:
\begin{equation}\label{eq:linparam}
w_j(k; \boldsymbol{\theta}_j) = \sum_{p=1}^{P_j} \theta_{j,p} \, \phi_p(k) = \boldsymbol{\phi}(k)^\top \boldsymbol{\theta}_j,
\end{equation}
where $\phi_1(\cdot), \ldots, \phi_{P_j}(\cdot)$ are predetermined basis functions and $\boldsymbol{\theta}_j \in \mathbb{R}^{P_j}$. Our primary specification uses \textbf{Almon polynomials}: $\phi_p(k) = k^{p-1}$. As a robustness check, we also employ cubic B-splines with uniformly spaced knots.

Using \eqref{eq:linparam}, the observation equation becomes
\begin{equation}\label{eq:bilinear}
y_t = \alpha + \sum_{j=1}^{J} \beta_j \, \mathbf{x}_t^{(j)\top} \boldsymbol{\Phi}_j \boldsymbol{\theta}_j + \varepsilon_t,
\end{equation}
where $\boldsymbol{\Phi}_j \in \mathbb{R}^{K_j \times P_j}$ is the basis matrix. The product $\beta_j \cdot (\mathbf{x}_t^{(j)\top} \boldsymbol{\Phi}_j \boldsymbol{\theta}_j)$ is \textbf{bilinear} in the unknown parameters~$\beta_j$ and~$\boldsymbol{\theta}_j$---the source of both the computational challenge and the opportunity.

\begin{remark}[Negative weights]
The linear parameterization permits negative weights for certain values of~$\boldsymbol{\theta}_j$, unlike nonlinear specifications. This is a trade-off: computational tractability in exchange for relaxed economic interpretability. In practice, estimated weight profiles are well-behaved, as confirmed by our Monte Carlo and empirical results.
\end{remark}

\subsection{Normalization constraint and reparameterization}

The MIDAS model in~\eqref{eq:bilinear} suffers from a scale indeterminacy: for any $c \neq 0$, the transformation $(\beta_j, \boldsymbol{\theta}_j) \mapsto (c\beta_j, \boldsymbol{\theta}_j / c)$ leaves the likelihood invariant. The standard resolution imposes that weights sum to unity:
\begin{equation}\label{eq:constraint}
\mathbf{c}_j^\top \boldsymbol{\theta}_j = 1, \qquad \text{where} \quad \mathbf{c}_j = \boldsymbol{\Phi}_j^\top \mathbf{1}_{K_j} \in \mathbb{R}^{P_j}.
\end{equation}
We handle this through a \textbf{null-space reparameterization} \citep{chan2025}. Any $\boldsymbol{\theta}_j$ satisfying~\eqref{eq:constraint} can be decomposed as
\begin{equation}\label{eq:reparam}
\boldsymbol{\theta}_j = \boldsymbol{\theta}_j^0 + \mathbf{N}_j \boldsymbol{\eta}_j,
\end{equation}
where $\boldsymbol{\theta}_j^0 = \mathbf{c}_j / \|\mathbf{c}_j\|^2$ is the minimum-norm particular solution, $\mathbf{N}_j \in \mathbb{R}^{P_j \times (P_j - 1)}$ has orthonormal columns spanning $\ker(\mathbf{c}_j^\top)$, and $\boldsymbol{\eta}_j \in \mathbb{R}^{P_j - 1}$ is the free parameter. Since $\mathbf{c}_j^\top \mathbf{N}_j = \mathbf{0}$, the constraint is automatically satisfied.

Substituting~\eqref{eq:reparam} into the aggregation yields
\begin{equation}\label{eq:reparam_agg}
\tilde{x}_t^{(j)} = \underbrace{\mathbf{x}_t^{(j)\top} \boldsymbol{\Phi}_j \boldsymbol{\theta}_j^0}_{c_t^{(j)}} + \underbrace{\mathbf{x}_t^{(j)\top} \boldsymbol{\Phi}_j \mathbf{N}_j}_{\mathbf{r}_t^{(j)\top}} \boldsymbol{\eta}_j,
\end{equation}
where $c_t^{(j)}$ is a known scalar and $\mathbf{r}_t^{(j)} \in \mathbb{R}^{P_j - 1}$ is the reduced MIDAS regressor. The reparameterized model is
\begin{equation}\label{eq:model_reparam}
y_t = \alpha + \sum_{j=1}^{J} \beta_j \left[ c_t^{(j)} + \mathbf{r}_t^{(j)\top} \boldsymbol{\eta}_j \right] + \varepsilon_t.
\end{equation}
This form makes conditional linearity transparent: given $\{\boldsymbol{\eta}_j\}$, the model is linear in $\boldsymbol{\xi} = (\alpha, \beta_1, \ldots, \beta_J)^\top$; given~$\boldsymbol{\xi}$, it is linear in each~$\boldsymbol{\eta}_j$.

\subsection{Prior specification}

We assign independent priors:
\begin{align}
\boldsymbol{\eta}_j &\sim \mathcal{N}(\mathbf{0}, \sigma_\eta^2 \, \mathbf{I}_{P_j-1}), \qquad \sigma_\eta^2 = 1, \label{eq:prior_eta} \\
\alpha &\sim \mathcal{N}(0, 100), \qquad \beta_j \sim \mathcal{N}(0, 10), \label{eq:prior_xi} \\
\sigma^2 &\sim \text{Inverse-Gamma}(0.01, 0.01). \label{eq:prior_sigma}
\end{align}
The prior on~$\boldsymbol{\eta}_j$ is centered at zero, corresponding to the weight profile induced by~$\boldsymbol{\theta}_j^0$. The priors on $\alpha$ and~$\beta_j$ are weakly informative, and the Inverse-Gamma prior on~$\sigma^2$ is the standard diffuse specification.

\begin{remark}[Choice of priors]
We use identical Gaussian priors for all~$\beta_j$ rather than shrinkage priors. This deliberate choice isolates algorithm performance from prior specification effects. In the Monte Carlo study, null coefficients ($\beta_j = 0$) are recovered with bias below 0.005 without shrinkage. Extension to global-local shrinkage priors is straightforward.
\end{remark}

\subsection{Posterior}

The joint posterior is
\begin{equation}\label{eq:posterior}
p(\boldsymbol{\xi}, \boldsymbol{\eta}_1, \ldots, \boldsymbol{\eta}_J, \sigma^2 \mid \mathbf{y}) \propto p(\mathbf{y} \mid \boldsymbol{\xi}, \boldsymbol{\eta}, \sigma^2) \, p(\boldsymbol{\xi}) \prod_{j=1}^{J} p(\boldsymbol{\eta}_j) \, p(\sigma^2).
\end{equation}
Due to the bilinear terms, the marginal likelihood has no closed form. However, the \textbf{conditional posteriors} are standard: $p(\boldsymbol{\xi} \mid \boldsymbol{\eta}, \sigma^2, \mathbf{y})$ is Gaussian, each $p(\boldsymbol{\eta}_j \mid \boldsymbol{\xi}, \boldsymbol{\eta}_{-j}, \sigma^2, \mathbf{y})$ is Gaussian, and $p(\sigma^2 \mid \boldsymbol{\xi}, \boldsymbol{\eta}, \mathbf{y})$ is Inverse-Gamma.

\begin{remark}[Failure of NUTS]
We attempted estimation using NUTS \citep{hoffman2014}, which operates on the joint parameter space without exploiting conditional structure. The bilinear product creates a funnel-like geometry \citep{neal2003} where posterior curvature varies by orders of magnitude. NUTS trajectories consistently produced numerical divergences, confirming that generic HMC is ill-suited to this model class.
\end{remark}

\section{Variational Inference for MIDAS Regression}

\subsection{Mean-field approximation}

We approximate the posterior~\eqref{eq:posterior} with
\begin{equation}\label{eq:meanfield}
q(\boldsymbol{\xi}, \boldsymbol{\eta}_1, \ldots, \boldsymbol{\eta}_J, \sigma^2) = q(\boldsymbol{\xi}) \prod_{j=1}^{J} q(\boldsymbol{\eta}_j) \, q(\sigma^2),
\end{equation}
chosen to maximize the evidence lower bound (ELBO),
\begin{equation}\label{eq:elbo}
\mathcal{L}(q) = \mathbb{E}_q[\log p(\mathbf{y}, \boldsymbol{\xi}, \boldsymbol{\eta}, \sigma^2)] - \mathbb{E}_q[\log q(\boldsymbol{\xi}, \boldsymbol{\eta}, \sigma^2)] \leq \log p(\mathbf{y}).
\end{equation}

The intercept and impact coefficients are kept in a \textbf{joint block}~$\boldsymbol{\xi}$, preserving their posterior correlations. The weight parameters $\boldsymbol{\eta}_j$ are separated from~$\boldsymbol{\xi}$ and from each other---this is the mean-field approximation that ignores cross-block correlations in exchange for tractability. Under the standard mean-field result \citep{blei2017}, the optimal factor $q_m^*$ satisfies
\begin{equation}\label{eq:cavi_rule}
\log q_m^*(\boldsymbol{\psi}_m) = \mathbb{E}_{-m}[\log p(\mathbf{y}, \boldsymbol{\psi})] + \text{const}.
\end{equation}

\subsection{Block 1: Regression coefficients}

Define the effective regressor vector $\tilde{\mathbf{x}}_t = (1, \tilde{x}_t^{(1)}, \ldots, \tilde{x}_t^{(J)})^\top \in \mathbb{R}^{J+1}$.

\begin{proposition}\label{prop:block1}
Under the mean-field factorization~\eqref{eq:meanfield}, the optimal variational distribution for the regression coefficients is $q^*(\boldsymbol{\xi}) = \mathcal{N}(\boldsymbol{\mu}_\xi, \boldsymbol{\Sigma}_\xi)$ with
\begin{align}
\boldsymbol{\Sigma}_\xi &= \left( \mathbb{E}_q[\sigma^{-2}] \sum_{t=1}^{T} \mathbb{E}_q[\tilde{\mathbf{x}}_t \tilde{\mathbf{x}}_t^\top] + \boldsymbol{\Lambda}_\xi \right)^{-1}, \label{eq:Sigma_xi} \\
\boldsymbol{\mu}_\xi &= \boldsymbol{\Sigma}_\xi \left( \mathbb{E}_q[\sigma^{-2}] \sum_{t=1}^{T} y_t \, \mathbb{E}_q[\tilde{\mathbf{x}}_t] \right), \label{eq:mu_xi}
\end{align}
where $\boldsymbol{\Lambda}_\xi = \mathrm{diag}(\sigma_\alpha^{-2}, \sigma_\beta^{-2}, \ldots, \sigma_\beta^{-2})$ is the prior precision matrix.
\end{proposition}

The required moments are $\mathbb{E}_q[\tilde{x}_t^{(j)}] = c_t^{(j)} + \mathbf{r}_t^{(j)\top} \boldsymbol{\mu}_{\eta_j}$ and
\begin{equation}\label{eq:second_moment}
\mathbb{E}_q[(\tilde{x}_t^{(j)})^2] = \left(\mathbb{E}_q[\tilde{x}_t^{(j)}]\right)^2 + \mathbf{r}_t^{(j)\top} \boldsymbol{\Sigma}_{\eta_j} \, \mathbf{r}_t^{(j)}.
\end{equation}
The \textbf{variance correction} term $\mathbf{r}_t^{(j)\top} \boldsymbol{\Sigma}_{\eta_j} \mathbf{r}_t^{(j)}$ propagates uncertainty about~$\boldsymbol{\eta}_j$ into the update of~$\boldsymbol{\xi}$, distinguishing this from a plug-in approach.

\emph{Proof.} Applying the CAVI rule~\eqref{eq:cavi_rule} to~$\boldsymbol{\xi}$ and collecting quadratic terms from $\mathbb{E}_{-\boldsymbol{\xi}}[\log p(\mathbf{y} \mid \boldsymbol{\xi}, \boldsymbol{\eta}, \sigma^2)]$ yields a Gaussian with the stated precision~\eqref{eq:Sigma_xi}. The second-moment matrix $\mathbb{E}_q[\tilde{\mathbf{x}}_t \tilde{\mathbf{x}}_t^\top]$ arises from completing the square in $\boldsymbol{\xi}$, where the off-diagonal entries couple $\alpha$ and each~$\beta_j$ through $\mathbb{E}_q[\tilde{x}_t^{(j)}]$, and the diagonal entries involve~\eqref{eq:second_moment}. The linear term in~$\boldsymbol{\xi}$ gives the mean~\eqref{eq:mu_xi}. \hfill$\square$

\subsection{Block 2: Weight parameters}

For each predictor $j = 1, \ldots, J$, define the partial residual
\begin{equation}\label{eq:partial_resid}
\bar{e}_t^{(-j)} = y_t - \mu_\alpha - \sum_{j' \neq j} \mu_{\beta_{j'}} \, \mathbb{E}_q[\tilde{x}_t^{(j')}] - \mu_{\beta_j} \, c_t^{(j)}.
\end{equation}

\begin{proposition}\label{prop:block2}
Under the mean-field factorization~\eqref{eq:meanfield}, the optimal variational distribution for the weight parameters of predictor~$j$ is $q^*(\boldsymbol{\eta}_j) = \mathcal{N}(\boldsymbol{\mu}_{\eta_j}, \boldsymbol{\Sigma}_{\eta_j})$ with
\begin{align}
\boldsymbol{\Sigma}_{\eta_j} &= \left( \mathbb{E}_q[\sigma^{-2}] \, \mathbb{E}_q[\beta_j^2] \sum_{t=1}^{T} \mathbf{r}_t^{(j)} \mathbf{r}_t^{(j)\top} + \sigma_\eta^{-2} \mathbf{I}_{P_j-1} \right)^{-1}, \label{eq:Sigma_eta} \\
\boldsymbol{\mu}_{\eta_j} &= \boldsymbol{\Sigma}_{\eta_j} \cdot \mathbb{E}_q[\sigma^{-2}] \left( \mu_{\beta_j} \sum_{t=1}^{T} \mathbf{r}_t^{(j)} \bar{e}_t^{(-j)} - \sum_{t=1}^{T} \mathbf{r}_t^{(j)} \, \mathbf{g}_t^\top \boldsymbol{\Sigma}_\xi \mathbf{e}_{j+1} \right), \label{eq:mu_eta}
\end{align}
where $\mathbf{g}_t = \mathbb{E}_q[\tilde{\mathbf{x}}_t]$, $\mathbf{e}_{j+1}$ is the $(j+1)$-th canonical basis vector, and $\mathbb{E}_q[\beta_j^2] = \mu_{\beta_j}^2 + (\boldsymbol{\Sigma}_\xi)_{j+1,j+1}$.
\end{proposition}

Three features of this update merit emphasis. \emph{First}, the precision~\eqref{eq:Sigma_eta} is scaled by $\mathbb{E}_q[\beta_j^2]$, which includes the \textbf{variance} of~$\beta_j$. When $\beta_j$ is close to zero, the prior dominates and weights revert to the baseline---automatic regularization of irrelevant predictors.

\emph{Second}, the mean~\eqref{eq:mu_eta} contains a \textbf{covariance correction} term $-\sum_t \mathbf{r}_t^{(j)} \mathbf{g}_t^\top \boldsymbol{\Sigma}_\xi \mathbf{e}_{j+1}$, arising because~$\beta_j$ is jointly distributed with other components of~$\boldsymbol{\xi}$. Omitting this correction would ignore the benefits of the joint $(\alpha, \boldsymbol{\beta})$ treatment.

\emph{Third}, the matrix inversion in~\eqref{eq:Sigma_eta} has dimension $(P_j - 1) \times (P_j - 1)$, typically 2 or 3, making the per-predictor update negligible.

\emph{Proof.} See Online Appendix~A.1.

\subsection{Block 3: Error variance}

\begin{proposition}\label{prop:block3}
Under the mean-field factorization, the optimal variational distribution for the error variance is $q^*(\sigma^2) = \mathrm{Inverse\text{-}Gamma}(\tilde{a}, \tilde{b})$ with
\begin{equation}\label{eq:sigma_update}
\tilde{a} = a_0 + T/2, \qquad \tilde{b} = b_0 + \tfrac{1}{2} \sum_{t=1}^{T} \mathbb{E}_q[e_t^2],
\end{equation}
where $\mathbb{E}_q[e_t^2] = y_t^2 - 2 y_t \mathbb{E}_q[\tilde{\mathbf{x}}_t]^\top \boldsymbol{\mu}_\xi + \mathrm{tr}\!\left(\mathbb{E}_q[\tilde{\mathbf{x}}_t \tilde{\mathbf{x}}_t^\top] \cdot (\boldsymbol{\mu}_\xi \boldsymbol{\mu}_\xi^\top + \boldsymbol{\Sigma}_\xi)\right)$.
\end{proposition}

\emph{Proof.} Applying the CAVI rule~\eqref{eq:cavi_rule} to~$\sigma^2$, the expected log-joint collects all terms involving~$\sigma^2$: the log-likelihood contributes $-\frac{T}{2}\log\sigma^2 - \frac{1}{2\sigma^2}\sum_t \mathbb{E}_q[e_t^2]$, and the Inverse-Gamma prior contributes $-(a_0+1)\log\sigma^2 - b_0/\sigma^2$. Combining yields the kernel of an Inverse-Gamma with the stated parameters~\eqref{eq:sigma_update}. The expected squared residual $\mathbb{E}_q[e_t^2]$ expands using $\mathbb{E}_q[\boldsymbol{\xi}\boldsymbol{\xi}^\top] = \boldsymbol{\mu}_\xi\boldsymbol{\mu}_\xi^\top + \boldsymbol{\Sigma}_\xi$. \hfill$\square$

The moments required by other blocks are $\mathbb{E}_q[\sigma^{-2}] = \tilde{a}/\tilde{b}$ and $\mathbb{E}_q[\log \sigma^2] = \log \tilde{b} - \psi(\tilde{a})$.

\subsection{The CAVI algorithm}

\begin{algorithm}[H]
\SetAlgoLined
\KwIn{Data $\{\mathbf{y}, \mathbf{x}_t^{(j)}\}$; basis matrices $\{\boldsymbol{\Phi}_j\}$; hyperparameters; tolerance $\varepsilon$}
\KwOut{Variational parameters $\boldsymbol{\mu}_\xi, \boldsymbol{\Sigma}_\xi$; $\{\boldsymbol{\mu}_{\eta_j}, \boldsymbol{\Sigma}_{\eta_j}\}$; $\tilde{a}, \tilde{b}$; ELBO trace}
\textbf{Preprocessing:} Compute $\boldsymbol{\theta}_j^0$, $\mathbf{N}_j$, $c_t^{(j)}$, $\mathbf{r}_t^{(j)}$ for all $j, t$\;
\textbf{Initialize:} $\boldsymbol{\mu}_\xi$ from OLS with uniform weights (see Appendix~A.4); $\boldsymbol{\mu}_{\eta_j} = \mathbf{0}$; $\tilde{a} = a_0 + T/2$\;
\Repeat{$|\Delta \mathcal{L}| / |\mathcal{L}| < \varepsilon$}{
  \For{$j = 1, \ldots, J$}{
    Update $q^*(\boldsymbol{\eta}_j)$ via \eqref{eq:Sigma_eta}--\eqref{eq:mu_eta}\;
  }
  Update $q^*(\boldsymbol{\xi})$ via \eqref{eq:Sigma_xi}--\eqref{eq:mu_xi}\;
  Update $q^*(\sigma^2)$ via \eqref{eq:sigma_update}\;
  Compute $\mathcal{L}(q)$ via \eqref{eq:elbo}\;
}
\caption{CAVI for Bayesian MIDAS Regression}
\label{alg:cavi}
\end{algorithm}

\begin{proposition}[Monotone convergence]\label{prop:convergence}
The ELBO sequence $\{\mathcal{L}^{(t)}\}$ generated by Algorithm~\ref{alg:cavi} satisfies $\mathcal{L}^{(t+1)} \geq \mathcal{L}^{(t)}$ for all~$t$, and converges to a fixed point.
\end{proposition}

\emph{Proof.} Each coordinate update maximizes the ELBO over the corresponding factor while holding all other factors fixed. Since the ELBO is bounded above by $\log p(\mathbf{y})$, the monotonically non-decreasing sequence converges. This is a standard property of CAVI with exponential family conditionals; see \citet{blei2017}, Section~2.4. \hfill$\square$

Across 10,500 CAVI fits in our Monte Carlo study, no ELBO violation was observed.

\subsection{Computational complexity}

The dominant cost per iteration is the $(J+1) \times (J+1)$ matrix inversion in Step~2, which is $O(J^3)$. The $J$ inversions in Step~1 are each $(P_j - 1) \times (P_j - 1)$---negligible for $P_j \leq 4$. Convergence requires 3--70 iterations (3 for $J=1$, up to 70 for $J=50$). With $T=200$, $J=50$, $P=3$, a complete CAVI run takes approximately 1.4 seconds on a standard laptop.

\subsection{Mean-field underdispersion}\label{sec:underdispersion}

The mean-field factorization ignores $\mathrm{Cov}(\beta_j, \boldsymbol{\eta}_j)$. This produces variational marginals more concentrated than the true posterior---underdispersion \citep{blei2017}. The effect is \textbf{asymmetric}: the impact coefficients in~$\boldsymbol{\xi}$ are correlated with all $J$ weight vectors simultaneously, while each~$\boldsymbol{\eta}_j$ is primarily correlated with its own~$\beta_j$. As $J$ grows, cumulative ignored correlations increase for~$\boldsymbol{\xi}$ but remain approximately constant for each~$\boldsymbol{\eta}_j$, explaining why coverage degrades for~$\beta$ but not for~$\boldsymbol{\eta}$ (Section~4).

A partial mitigation operates through the second-moment coupling: $\mathbb{E}_q[\beta_j^2]$ in~\eqref{eq:Sigma_eta} propagates the uncertainty of~$\beta_j$ into the calibration of~$\boldsymbol{\eta}_j$. No analogous mechanism compensates~$\beta_j$ at the individual coefficient level.

\section{Monte Carlo Simulation Study}

\subsection{Design}

We evaluate CAVI against a block Gibbs sampler across 21 DGP configurations organized in three tiers, with 500 replications each (10,500 CAVI fits; 9,500 Gibbs fits---the Gibbs sampler is omitted for $J \in \{25, 50\}$ due to prohibitive computation time). Tier~1 varies $J \in \{1,3,5,10,25,50\}$ and $T \in \{50,100,200,400\}$; Tier~2 examines weight profiles, signal-to-noise ratios, basis functions, and frequency ratios; Tier~3 contains stress tests. For configurations with $J > 1$, half the impact coefficients are zero.

\subsection{Main results: scalability in $J$}

Table~\ref{tab:scalability} reports results as $J$ increases from 1 to 50, with $T=200$, $K=9$, $P=3$.

\begin{table}[H]
\centering
\caption{CAVI versus Gibbs across number of predictors $J$. Metrics averaged over active $\beta_j$, 500 replications (50 for $J=25$ Gibbs). Standard errors in parentheses. $T=200$, $K=9$, $P=3$.}
\label{tab:scalability}
\footnotesize
\setlength{\tabcolsep}{3.5pt}
\begin{tabular}{@{}cccccccccc@{}}
\toprule
$J$ & Method & Bias($\beta$) & RMSE($\beta$) & Cov95($\beta$) & Bias($\eta$) & Cov95($\eta$) & Time (s) & Speedup & ESS \\
\midrule
1 & CAVI  & 0.029 \tiny{(.007)} & 0.158 \tiny{(.004)} & 0.894 \tiny{(.014)} & 0.001 \tiny{(.001)} & 0.929 \tiny{(.011)} & 0.001 & 1,772$\times$ & --- \\
1 & Gibbs & 0.032 \tiny{(.007)} & 0.159 \tiny{(.004)} & 0.934 \tiny{(.011)} & 0.002 \tiny{(.001)} & 0.936 \tiny{(.011)} & 1.82  & --- & 4,103 \\
3 & CAVI  & 0.056 \tiny{(.004)} & 0.173 \tiny{(.003)} & 0.836 \tiny{(.017)} & 0.005 \tiny{(.002)} & 0.945 \tiny{(.010)} & 0.006 & 645$\times$ & --- \\
3 & Gibbs & 0.078 \tiny{(.004)} & 0.187 \tiny{(.003)} & 0.903 \tiny{(.013)} & 0.008 \tiny{(.002)} & 0.957 \tiny{(.009)} & 3.59  & --- & 1,066 \\
5 & CAVI  & 0.087 \tiny{(.003)} & 0.177 \tiny{(.002)} & 0.594 \tiny{(.022)} & 0.030 \tiny{(.003)} & 0.953 \tiny{(.009)} & 0.023 & 283$\times$ & --- \\
5 & Gibbs & 0.103 \tiny{(.003)} & 0.181 \tiny{(.002)} & 0.846 \tiny{(.016)} & 0.032 \tiny{(.002)} & 0.974 \tiny{(.007)} & 6.54  & --- & 701 \\
10 & CAVI & 0.079 \tiny{(.003)} & 0.178 \tiny{(.002)} & 0.602 \tiny{(.022)} & 0.050 \tiny{(.003)} & 0.963 \tiny{(.010)} & 0.066 & 238$\times$ & --- \\
10 & Gibbs & 0.096 \tiny{(.003)} & 0.181 \tiny{(.002)} & 0.850 \tiny{(.016)} & 0.047 \tiny{(.002)} & 0.986 \tiny{(.007)} & 15.56 & --- & 539 \\
25 & CAVI & 0.086 \tiny{(.003)} & 0.179 \tiny{(.002)} & 0.581 \tiny{(.022)} & 0.073 \tiny{(.003)} & 0.971 \tiny{(.010)} & 0.322 & 107$\times$ & --- \\
25 & Gibbs & 0.159 & 0.186 & 0.828 & --- & --- & 34.44 & --- & 134 \\
50 & CAVI & 0.103 \tiny{(.003)} & 0.196 \tiny{(.003)} & 0.550 \tiny{(.022)} & 0.078 \tiny{(.002)} & 0.980 \tiny{(.009)} & 1.383 & --- & --- \\
\bottomrule
\end{tabular}
\end{table}

\textbf{Point estimates.} Bias and RMSE are comparable between CAVI and Gibbs across all configurations where both are available. The difference in bias is at most 0.03, confirming the variational posterior mean as a reliable point estimator.

\textbf{Speed.} Speedup ranges from $107\times$ ($J=25$) to $1{,}772\times$ ($J=1$), as visualized in Figure~\ref{fig:speedup}. The advantage grows with dimension: at $J=25$, CAVI completes in 0.3 seconds versus 34 seconds for Gibbs. At $J=50$, CAVI completes in 1.4 seconds while the Gibbs sampler would require an estimated 60+ seconds (Figure~\ref{fig:timing}).

\textbf{Coverage of $\boldsymbol{\eta}$.} Coverage remains \textbf{above 92\%} across all configurations, reaching 98\% at $J=50$. MIDAS weight profiles estimated by CAVI are well-calibrated.

\textbf{Coverage of $\beta$.} Coverage declines from 89.4\% ($J=1$) to 55.0\% ($J=50$), concentrated at $J \geq 5$. Two contextual observations: (i) Gibbs itself achieves only 85--93\% rather than 95\%, reflecting the challenging posterior geometry; (ii) the coverage gap reflects narrower intervals, not biased centers.

\textbf{Gibbs ESS degradation.} The minimum ESS declines from 4,103 to 539 as $J$ grows from 1 to 10, and further to 134 at $J=25$ (Figure~\ref{fig:ess}), indicating that fewer than 5\% of draws are effectively independent at $J=25$.

\begin{figure}[t]
\centering
\includegraphics[width=0.65\textwidth]{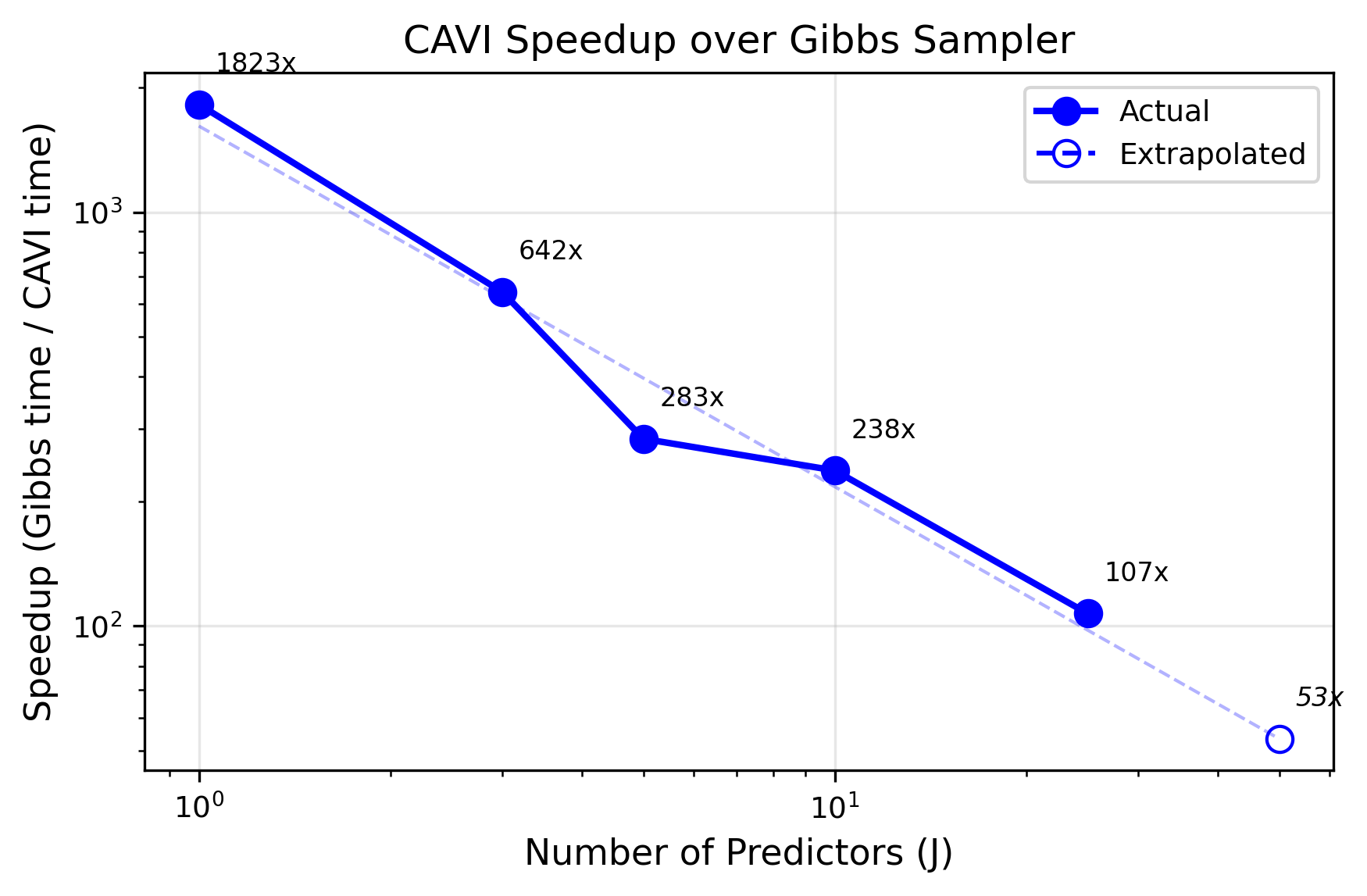}
\caption{Computational speedup of CAVI over Gibbs sampler as a function of the number of predictors~$J$. Solid markers: measured. Open markers: extrapolated from Gibbs scaling pattern. Even at $J=50$, CAVI maintains a speedup of approximately $53\times$ (extrapolated).}
\label{fig:speedup}
\end{figure}

\begin{figure}[t]
\centering
\includegraphics[width=0.65\textwidth]{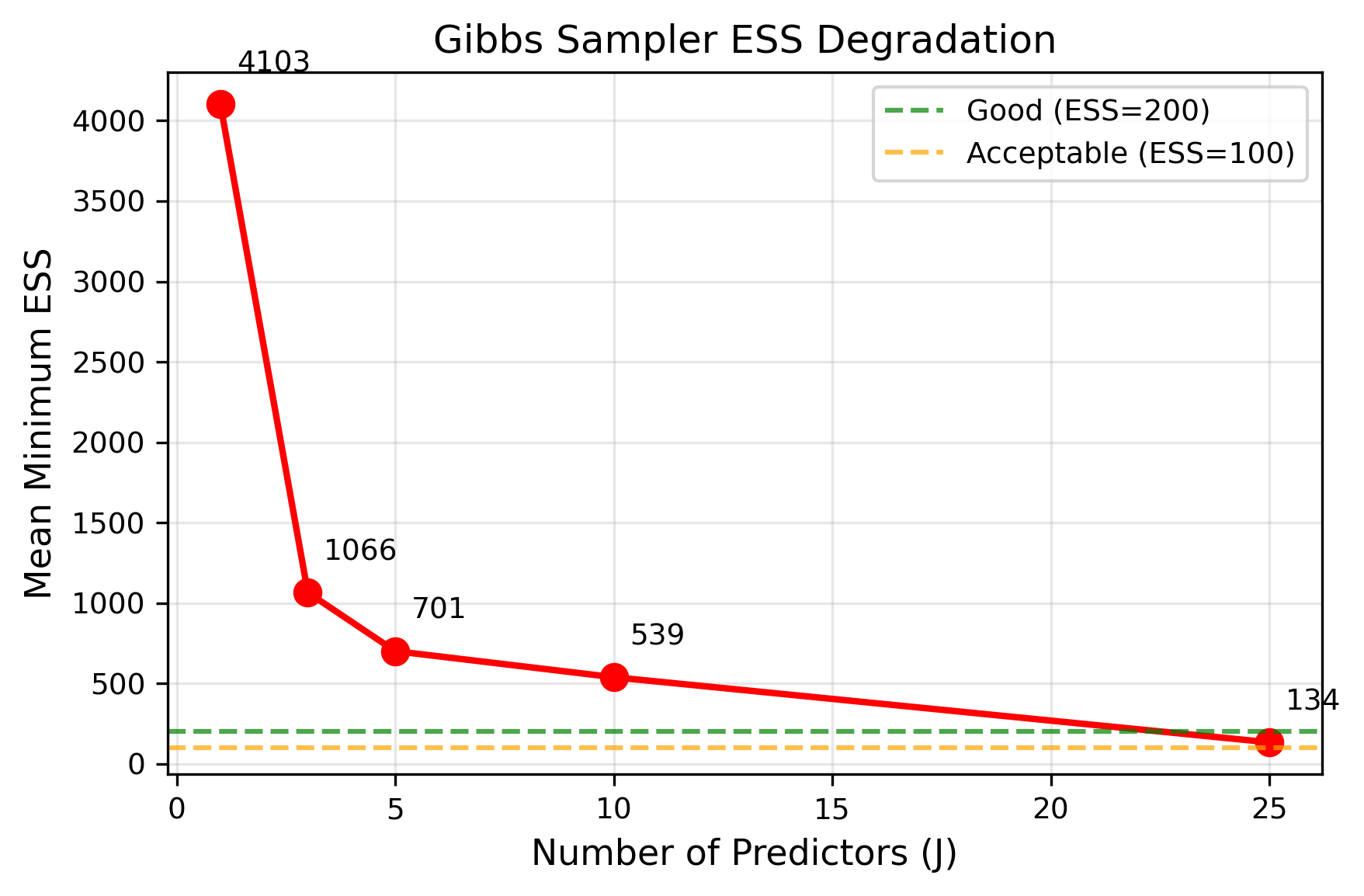}
\caption{Gibbs sampler minimum effective sample size (ESS) as a function of~$J$. The decline from 4,103 ($J=1$) to 539 ($J=10$) and further to 134 ($J=25$, 50 replications) indicates increasing autocorrelation in the MCMC chain as bilinear coupling involves more blocks.}
\label{fig:ess}
\end{figure}

\subsection{Main results: sample size $T$}

Table~\ref{tab:samplesize} reports results as $T$ varies from 50 to 400, with $J=3$.

\begin{table}[H]
\centering
\caption{CAVI versus Gibbs across sample size $T$. 500 replications. Standard errors in parentheses. $J=3$, $K=9$, $P=3$.}
\label{tab:samplesize}
\footnotesize
\setlength{\tabcolsep}{3.5pt}
\begin{tabular}{@{}cccccccccc@{}}
\toprule
$T$ & Method & Bias($\beta$) & RMSE($\beta$) & Cov95($\beta$) & Bias($\eta$) & Cov95($\eta$) & Time (s) & Speedup & ESS \\
\midrule
50  & CAVI  & 0.216 \tiny{(.009)} & 0.425 \tiny{(.007)} & 0.584 \tiny{(.022)} & 0.023 \tiny{(.003)} & 0.931 \tiny{(.011)} & 0.015 & 303$\times$ & --- \\
50  & Gibbs & 0.278 \tiny{(.009)} & 0.444 \tiny{(.007)} & 0.807 \tiny{(.018)} & 0.034 \tiny{(.002)} & 0.977 \tiny{(.007)} & 4.47  & --- & 681 \\
100 & CAVI  & 0.130 \tiny{(.006)} & 0.274 \tiny{(.004)} & 0.729 \tiny{(.020)} & 0.010 \tiny{(.003)} & 0.950 \tiny{(.010)} & 0.011 & 416$\times$ & --- \\
100 & Gibbs & 0.170 \tiny{(.006)} & 0.297 \tiny{(.004)} & 0.860 \tiny{(.016)} & 0.009 \tiny{(.002)} & 0.975 \tiny{(.007)} & 4.39  & --- & 794 \\
200 & CAVI  & 0.056 \tiny{(.004)} & 0.173 \tiny{(.003)} & 0.836 \tiny{(.017)} & 0.005 \tiny{(.002)} & 0.945 \tiny{(.010)} & 0.006 & 645$\times$ & --- \\
200 & Gibbs & 0.078 \tiny{(.004)} & 0.187 \tiny{(.003)} & 0.903 \tiny{(.013)} & 0.008 \tiny{(.002)} & 0.957 \tiny{(.009)} & 3.59  & --- & 1,066 \\
400 & CAVI  & 0.028 \tiny{(.003)} & 0.118 \tiny{(.002)} & 0.869 \tiny{(.015)} & 0.003 \tiny{(.001)} & 0.932 \tiny{(.011)} & 0.003 & 1,504$\times$ & --- \\
400 & Gibbs & 0.036 \tiny{(.003)} & 0.125 \tiny{(.002)} & 0.929 \tiny{(.012)} & 0.009 \tiny{(.001)} & 0.940 \tiny{(.011)} & 4.60  & --- & 2,022 \\
\bottomrule
\end{tabular}
\end{table}

CAVI coverage on~$\beta$ increases from 58.4\% ($T=50$) to 86.9\% ($T=400$), approaching the Gibbs benchmark. At $T=400$ the gap narrows to 6 percentage points. The speedup is largest when $T$ is large: $1{,}504\times$ at $T=400$, because CAVI converges in only 5 iterations while Gibbs remains at 5,000 draws. Figure~\ref{fig:coverage} visualizes the coverage convergence.

\begin{figure}[t]
\centering
\includegraphics[width=0.65\textwidth]{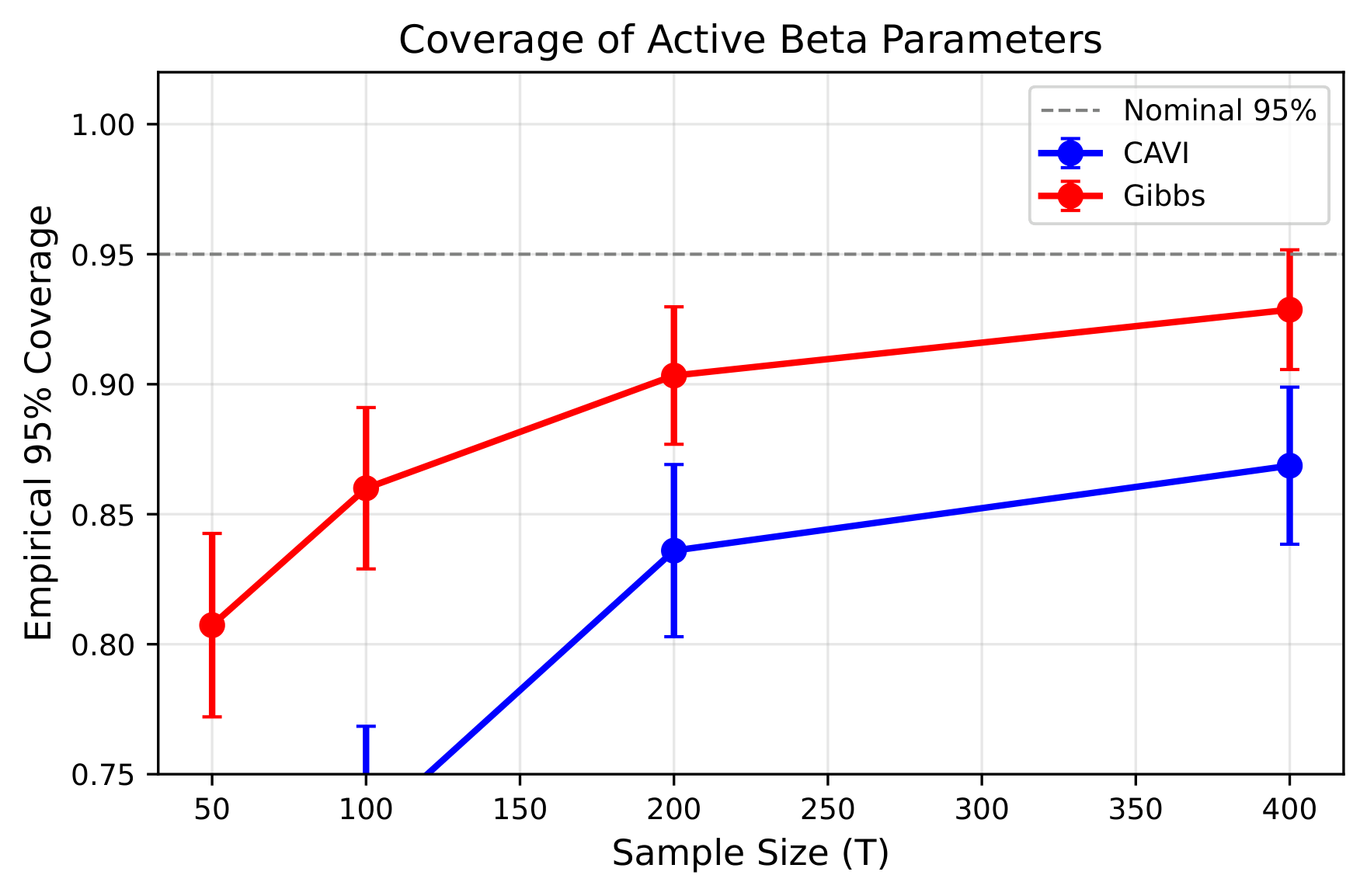}
\caption{Empirical 95\% coverage of active $\beta_j$ parameters as a function of sample size~$T$, with $J=3$. Error bars show $\pm 1$ standard error across replications. Both methods improve with~$T$. The CAVI--Gibbs gap narrows from approximately 10 percentage points at $T=100$ to 6 percentage points at $T=400$. The $T=50$ CAVI point (58.4\%, Table~\ref{tab:samplesize}) falls below the displayed range.}
\label{fig:coverage}
\end{figure}

\subsection{Robustness and stress tests}

\textbf{Weight profiles} (configs 2A-1 to 2A-3). Decreasing, hump-shaped, and U-shaped profiles are recovered accurately by both methods. Figure~\ref{fig:weights_mc} illustrates the recovery for the baseline $J=3$ configuration, confirming that CAVI and Gibbs produce nearly identical weight estimates across distinct profile shapes. Coverage on~$\eta$ exceeds 92\% for all shapes.

\textbf{Signal-to-noise ratio.} High SNR ($\sigma_0^2 = 0.25$) yields 90\% coverage on~$\beta$ for CAVI; low SNR ($\sigma_0^2 = 4$) reduces it to 75\%. Gibbs coverage also declines (95\% to 86\%), indicating model-inherent difficulty.

\textbf{Frequency ratio.} At $K=65$ (quarterly-to-daily), both CAVI and Gibbs struggle: $\beta$ coverage drops to 48--50\% for both, confirming a DGP/basis limitation rather than an algorithm failure.

Full robustness and stress test tables are in Online Appendix~A.3.

\begin{figure}[t]
\centering
\includegraphics[width=\textwidth]{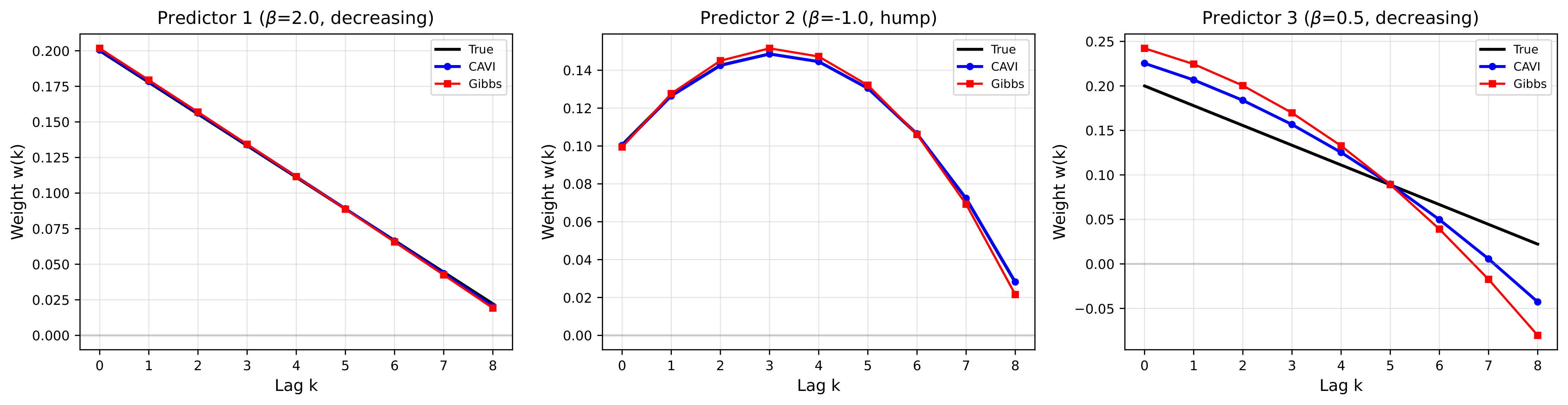}
\caption{Weight profile recovery for a representative replication ($J=3$, $T=200$). Left: decreasing profile ($\beta_1 = 2.0$). Center: hump-shaped profile ($\beta_2 = -1.0$). Right: decreasing profile ($\beta_3 = 0.5$). CAVI (blue) and Gibbs (red) estimates are nearly indistinguishable from the true profile (black) for strong predictors.}
\label{fig:weights_mc}
\end{figure}

\subsection{Comparison with automatic differentiation VI}

A natural question is whether generic Automatic Differentiation Variational Inference (ADVI; \citealp{kucukelbir2017}) could replace our model-specific CAVI. We compare against both mean-field ADVI and full-rank ADVI implemented in PyMC~5 across $J \in \{1, 3, 5\}$ with 5 replications each (10,000 gradient steps per fit).

\begin{table}[H]
\centering
\caption{CAVI versus generic ADVI (mean-field and full-rank). 5 replications per configuration. $T=200$, $K=9$, $P=3$.}
\label{tab:advi}
\small
\begin{tabular}{@{}clcccc@{}}
\toprule
$J$ & Method & Bias($\beta$) & RMSE($\beta$) & Time (s) & Speedup vs ADVI-MF \\
\midrule
1 & CAVI    & 0.128 & 0.235 & 0.000 & $>$100,000$\times$ \\
1 & ADVI-MF & 1.851 & 1.851 & 136.0 & --- \\
1 & ADVI-FR & 1.972 & 1.972 & 137.3 & --- \\
\addlinespace
3 & CAVI    & 0.128 & 0.157 & 0.009 & 24,644$\times$ \\
3 & ADVI-MF & 1.381 & 1.457 & 221.8 & --- \\
3 & ADVI-FR & 1.467 & 1.547 & 302.0 & --- \\
\addlinespace
5 & CAVI    & 0.164 & 0.226 & 0.022 & 1,934$\times$ \\
5 & ADVI-MF & 1.074 & 1.219 &  42.6 & --- \\
5 & ADVI-FR & 1.149 & 1.302 &  61.9 & --- \\
\bottomrule
\end{tabular}
\end{table}

CAVI dominates ADVI on both dimensions. \emph{Accuracy}: ADVI bias is 7--14 times larger than CAVI across all configurations, indicating that stochastic gradient optimization fails to find the correct posterior mode for the MIDAS bilinear structure. The ADVI loss function landscapes are likely plagued by the same funnel geometry that defeats NUTS. \emph{Speed}: CAVI is 2,000--100,000 times faster. The advantage is even more dramatic than the CAVI--Gibbs comparison because ADVI requires thousands of gradient evaluations through automatic differentiation, whereas CAVI exploits closed-form conditional conjugacy.

These results confirm that the contribution of this paper is not merely computational speedup over MCMC, but the derivation of model-specific updates that exploit conditional structure unavailable to black-box methods.

\subsection{Discussion}

The Monte Carlo evidence supports four conclusions. First, CAVI is a reliable estimator of posterior means: bias and RMSE match Gibbs across all configurations, including $J=25$ where the Gibbs benchmark is now available (Table~\ref{tab:scalability}). Second, CAVI dominates generic ADVI by a wide margin on both accuracy and speed (Table~\ref{tab:advi}), confirming the value of model-specific variational derivations. Third, weight function parameters are well-calibrated (coverage ${}>{}$92\%). Fourth, impact coefficient credible intervals exhibit mean-field underdispersion---moderate for $J \leq 3$ ($\geq$84\%) and substantial for $J \geq 5$ ($\approx$55--60\%).

For applications focused on point forecasting, variable selection, or weight profile estimation, CAVI offers a compelling alternative. For well-calibrated credible intervals on~$\beta$ in high-dimensional settings, Gibbs sampling or structured VI \citep{loaiza2022} should be preferred.

\paragraph{Practical coverage guidance.} For practitioners requiring uncertainty quantification on~$\beta$, we offer the following guidance based on the full Monte Carlo evidence. For $J \leq 3$ with $T \geq 200$, CAVI credible intervals achieve coverage above 83\%, adequate for most applications. A simple post-hoc calibration---inflating the variational standard deviation by a multiplier $\kappa$---restores nominal coverage for $J \leq 3$: a multiplier of $\kappa = 1.2$ suffices for $J=1$ (calibrated coverage 94.2\%) and $\kappa = 1.8$ for $J=3$ (calibrated coverage 94.9\%). For $J \geq 5$, the mean-field approximation becomes inadequate for interval estimation: even a multiplier of $\kappa = 3.0$ only reaches 86\% coverage at $J=5$ and 81\% at $J=50$, because the underdispersion is non-uniform across parameters. In such settings, we recommend using the Gibbs sampler for inference on~$\beta$ or, when computational constraints preclude MCMC, adopting the linear response corrections of \citet{giordano2018}, which adjust variational covariances using second-order sensitivity of the ELBO. Extension to structured variational families that capture $\mathrm{Cov}(\beta_j, \boldsymbol{\eta}_j)$ is a natural direction for future work.

\section{Empirical Application: Realized Volatility Forecasting}

\subsection{Setup}

We apply the MIDAS-RV specification of \citet{ghysels2006} to forecast monthly realized volatility (RV) of the S\&P~500 index:
\begin{equation}\label{eq:midasrv}
\log RV_t = \alpha + \beta \sum_{k=0}^{K-1} w(k; \boldsymbol{\theta}) \, r_{t-k}^2 + \varepsilon_t,
\end{equation}
where $RV_t = \sum_{d=1}^{D_t} r_{t,d}^2$ is the sum of squared daily returns within month~$t$. We use $K = 22$ daily lags and Almon polynomials with $P = 3$, and consider $J = 1$ and $J = 3$ specifications. In the $J=3$ specification, the three predictors correspond to daily squared returns from the current month ($t$), the previous month ($t{-}1$), and two months prior ($t{-}2$), each aggregated with its own MIDAS weight function. Benchmarks include HAR-RV \citep{corsi2009}, AR(1), AR(4), and the historical average.

Figure~\ref{fig:rvdata} displays the S\&P~500 monthly realized volatility series from January 2000 through December 2025 (312 months, 6,537 trading days). The dashed vertical line marks the start of the out-of-sample evaluation period (187 months, expanding window from 120-month initial training sample).

\begin{figure}[H]
\centering
\includegraphics[width=\textwidth]{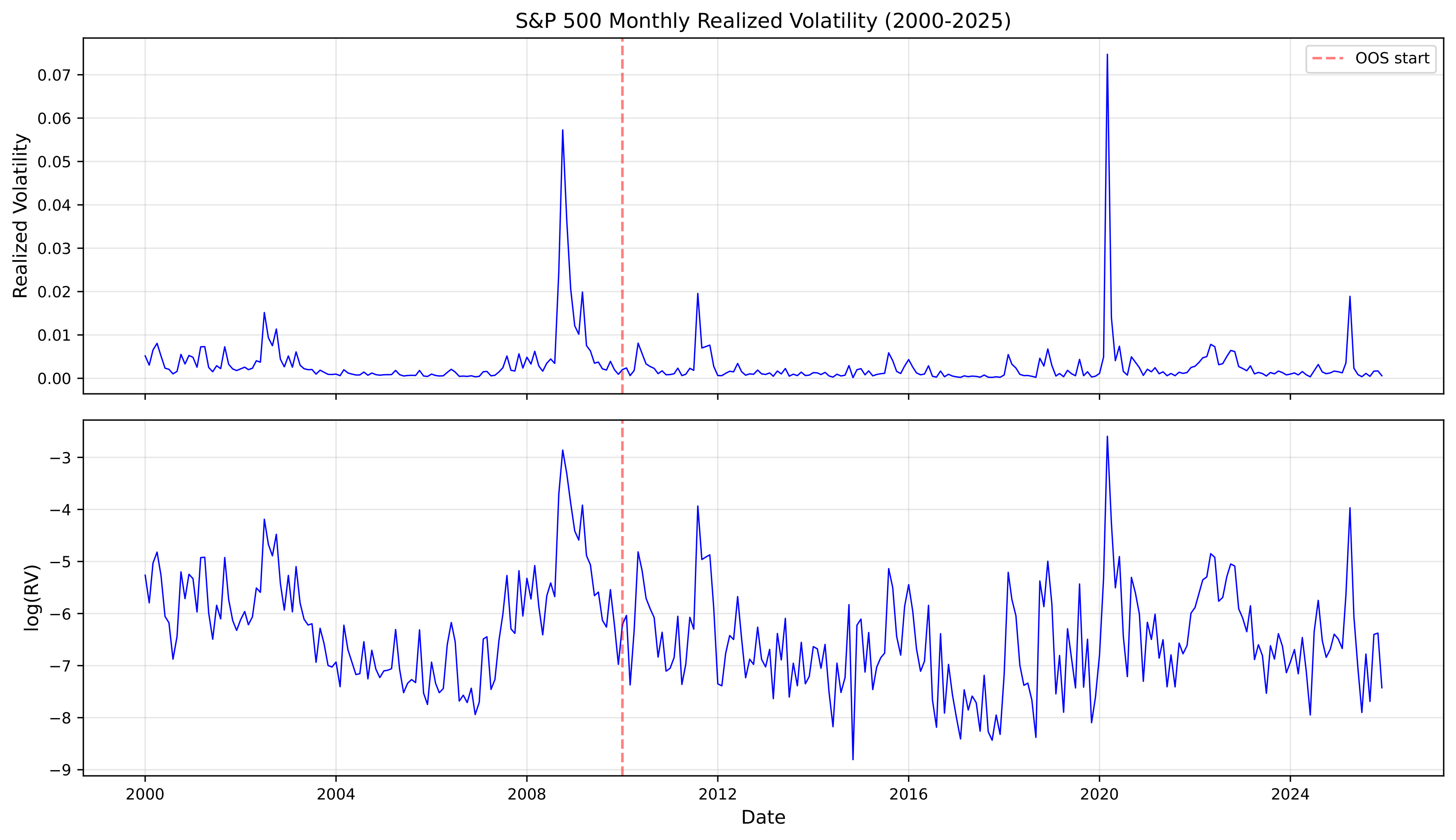}
\caption{S\&P~500 monthly realized volatility (2000--2025). Top: raw $RV_t$. Bottom: $\log RV_t$. Dashed line: start of out-of-sample period. Prominent spikes correspond to the 2008 financial crisis and the 2020 COVID-19 shock.}
\label{fig:rvdata}
\end{figure}

\subsection{Forecast accuracy}

Table~\ref{tab:forecast} reports out-of-sample forecast accuracy. Diebold-Mariano tests \citep{diebold1995} are against HAR-RV.

\begin{table}[H]
\centering
\caption{Out-of-sample forecast accuracy (187 months). MSE and MAE on $\log RV_t$.}
\label{tab:forecast}
\small
\begin{tabular}{@{}lccccc@{}}
\toprule
Model & MSE & MAE & Rel.~MSE vs HAR & DM $p$-value & Time/month \\
\midrule
MIDAS $J\!=\!1$ (CAVI)  & 0.720 & 0.651 & 1.058 & 0.226 & 0.001 s \\
MIDAS $J\!=\!1$ (Gibbs) & 0.720 & 0.651 & 1.058 & 0.224 & 1.13 s \\
MIDAS $J\!=\!3$ (CAVI)  & 0.688 & 0.633 & 1.011 & 0.864 & 0.009 s \\
MIDAS $J\!=\!3$ (Gibbs) & 0.702 & 0.635 & 1.031 & 0.633 & 2.48 s \\
HAR-RV                   & 0.680 & 0.645 & 1.000 & ---   & $<$0.001 s \\
AR(1)                    & 0.713 & 0.657 & 1.048 & 0.045 & $<$0.001 s \\
AR(4)                    & 0.711 & 0.653 & 1.046 & 0.358 & $<$0.001 s \\
Historical Avg           & 0.997 & 0.802 & 1.465 & 0.000 & $<$0.001 s \\
\bottomrule
\end{tabular}
\end{table}

\textbf{CAVI--Gibbs agreement.} The MSE difference between CAVI and Gibbs is 0.00006 for $J=1$ (relative difference: 0.01\%), confirming the Monte Carlo finding. Figure~\ref{fig:forecasts} shows that MIDAS~$J=1$ (CAVI) tracks the actual log-RV series closely, with performance comparable to HAR-RV.

\textbf{Forecasting performance.} MIDAS $J=3$ nearly matches HAR-RV (relative MSE~$=1.011$; DM $p=0.864$). Neither MIDAS specification is significantly worse than HAR at conventional levels, while AR(1) is significantly inferior ($p=0.045$) and the historical average is strongly dominated ($p<0.001$).

\textbf{Computational time.} CAVI completes each monthly estimation in 1~ms ($J=1$) or 9~ms ($J=3$), yielding speedups of $936\times$ and $264\times$ respectively. In a single-pass estimation such as this application, the absolute time difference (1--2 seconds for Gibbs) may seem inconsequential. However, the speed advantage becomes decisive in three settings that are increasingly common in applied work: (i)~large-scale model comparison exercises scanning over many predictor sets, lag structures, or basis specifications; (ii)~real-time nowcasting pipelines with hundreds of candidate predictors, where estimation must complete within a publication cycle; and (iii)~expanding-window or rolling-window evaluations with hundreds of re-estimations, as in the present exercise (187 out-of-sample months $\times$ multiple specifications). With $J=10$ predictors and 500 windows, CAVI completes in approximately 33 seconds versus 2.2 hours for Gibbs.

\begin{figure}[t]
\centering
\includegraphics[width=\textwidth]{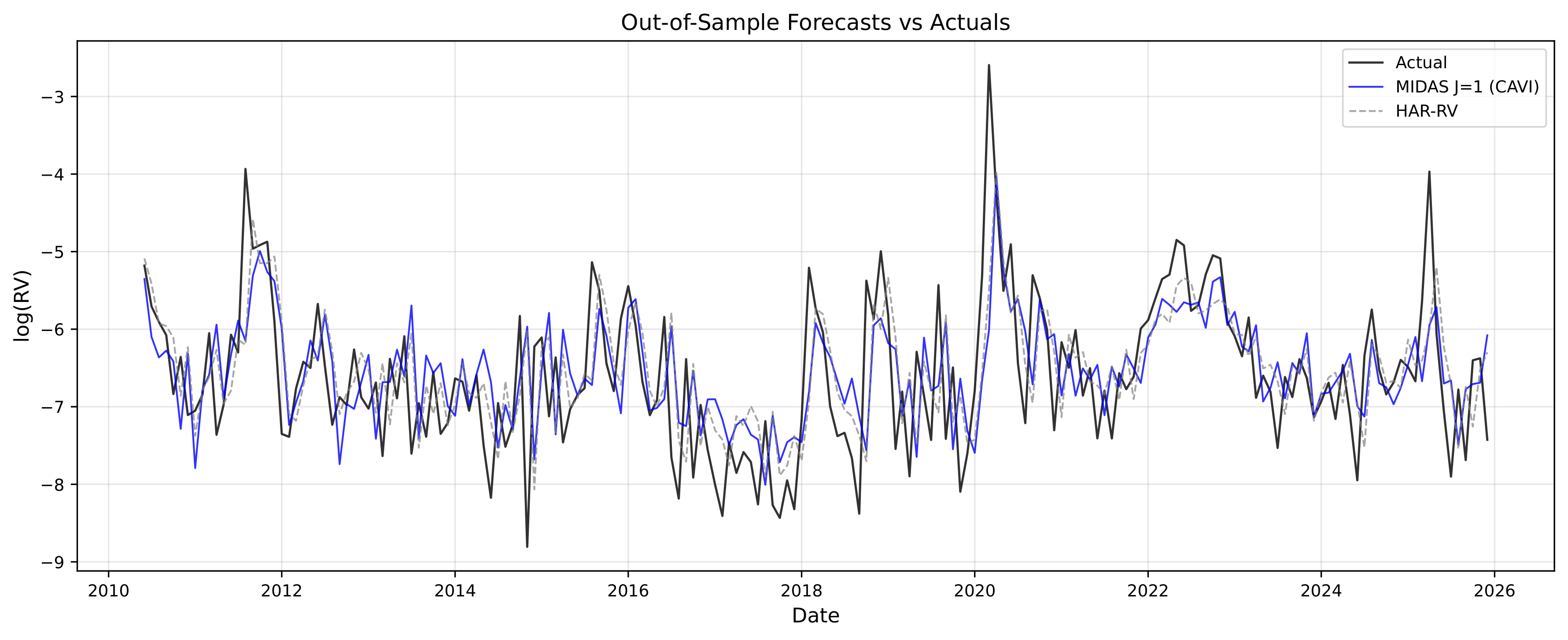}
\caption{Out-of-sample forecasts of $\log RV_t$ (2010--2025). Black: actual. Blue: MIDAS $J=1$ (CAVI). Gray dashed: HAR-RV. Both models track the actual series closely, with large forecast errors concentrated around the March 2020 COVID-19 volatility spike.}
\label{fig:forecasts}
\end{figure}

\subsection{Estimated weight profiles}

Figure~\ref{fig:weights_emp} displays the estimated MIDAS weight profile from the full-sample $J=1$ estimation. Both CAVI and Gibbs produce a monotonically decreasing profile, assigning greatest weight to the most recent daily observations---economically sensible, as recent squared returns are more informative about current-month volatility. The shaded regions show 95\% credible bands. The CAVI band (purple) is wider than the Gibbs band (pink), reflecting the different uncertainty propagation mechanisms: CAVI inflates the effective variance through the $\mathbb{E}_q[\beta^2]$ term in~\eqref{eq:Sigma_eta}, which includes the variance of~$\beta$, while the Gibbs band is computed conditionally on each posterior draw of~$\beta$. The point estimates are virtually identical.

The $J=3$ specification tells a complementary story (Figure~\ref{fig:weights_j3} in the Appendix): Predictor~1 (current month) retains the same decreasing profile, while Predictors~2 and~3 (months $t{-}1$ and $t{-}2$) have weights concentrated near zero with wide credible bands, indicating negligible marginal contribution---consistent with the modest forecasting gain of $J=3$ over $J=1$ (Table~\ref{tab:forecast}).

\begin{figure}[t]
\centering
\includegraphics[width=0.7\textwidth]{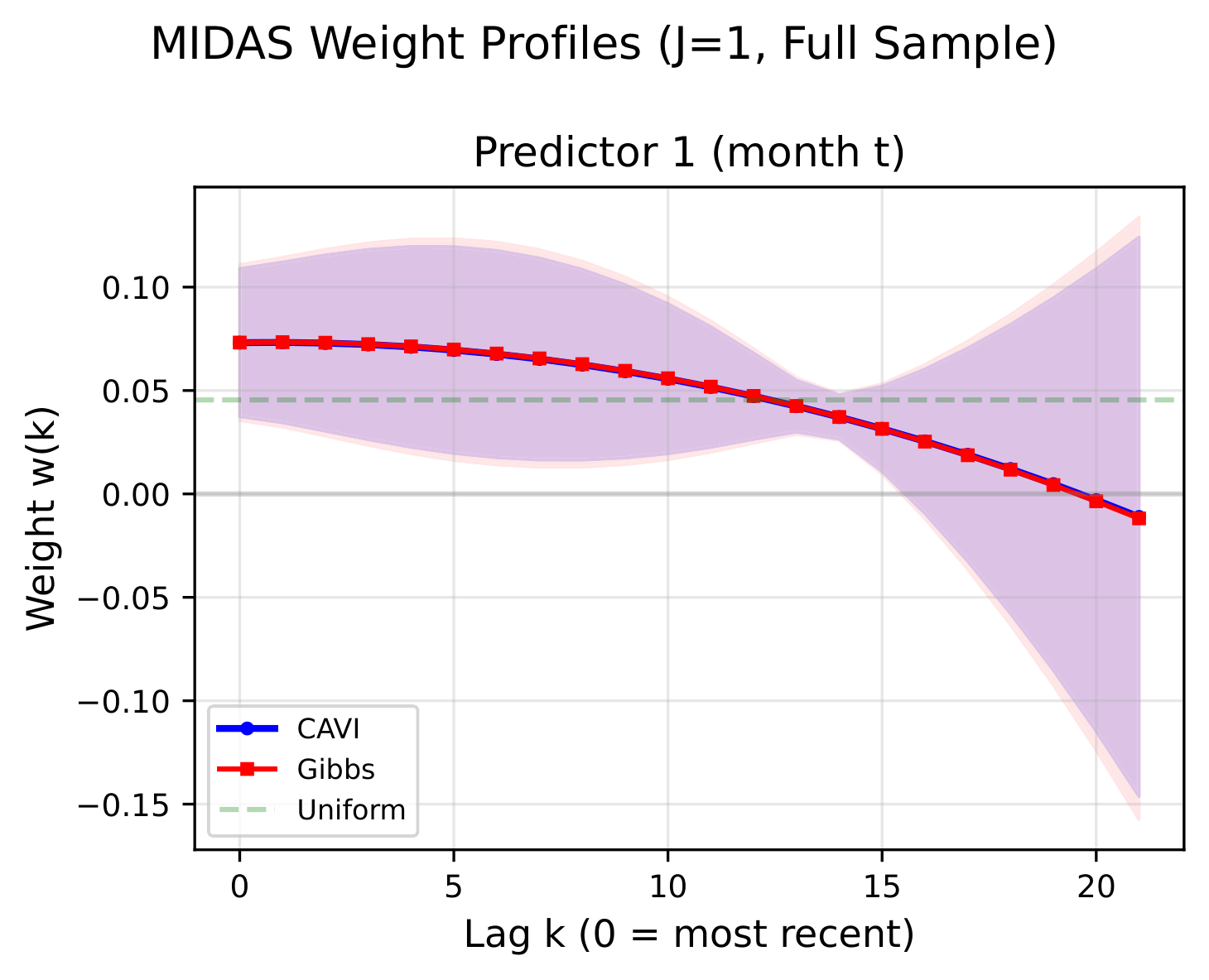}
\caption{Estimated MIDAS weight profiles ($J=1$, full sample). Blue: CAVI posterior mean. Red: Gibbs posterior mean. Shaded regions: 95\% credible bands. Dashed green: uniform weights. Point estimates are virtually identical. The monotonically decreasing profile assigns greater weight to recent daily returns, consistent with the intuition that recent squared returns are more informative about current-month volatility.}
\label{fig:weights_emp}
\end{figure}

\section{Conclusion}

This paper develops the first variational inference algorithm for Bayesian MIDAS regression. The CAVI algorithm exploits conditional conjugacy of the linearly parameterized model to produce closed-form updates for all parameter blocks, propagating uncertainty across blocks through second moments and including a covariance correction specific to the MIDAS bilinear setting.

A Monte Carlo study across 21 configurations demonstrates that CAVI posterior means are essentially identical to those from a block Gibbs sampler---now confirmed through $J=25$---while achieving speedups of $107\times$ to $1{,}772\times$ (Table~\ref{tab:timing}). Generic ADVI, by contrast, produces bias 7--14 times larger while being orders of magnitude slower, confirming that model-specific derivations are essential for this model class. Weight parameters are well-calibrated (coverage above 92\%). Impact coefficient credible intervals exhibit mean-field underdispersion, with coverage declining from 89\% to 55\% as predictor dimension grows---an honestly documented trade-off.

Three directions for future work follow naturally. First, extension to time-varying parameter MIDAS requires structured variational approximations \citep{loaiza2022} that capture temporal dependence---for example, targeting GDP nowcasting applications where parameter drift is economically motivated. Second, global-local shrinkage priors (Bayesian Lasso, horseshoe, Dirichlet-Laplace) can be incorporated with minimal additional derivation, enabling principled variable selection in large MIDAS panels. Third, stochastic volatility can be accommodated through the Gaussian approximation methods of \citet{chanyu2022}.


\newpage
\appendix
\section*{Online Appendix}
\setcounter{section}{0}
\renewcommand{\thesection}{A.\arabic{section}}

\section{Proof of Proposition~\ref{prop:block2}}

Starting from the CAVI rule~\eqref{eq:cavi_rule}, the log-optimal density for~$\boldsymbol{\eta}_j$ involves isolating terms dependent on~$\boldsymbol{\eta}_j$ from the expected log-joint. Writing $e_t = e_t^{(-j)} - \beta_j \mathbf{r}_t^{(j)\top}\boldsymbol{\eta}_j$ and expanding:

The quadratic term in~$\boldsymbol{\eta}_j$ gives
\[
-\tfrac{1}{2}\mathbb{E}_q[\sigma^{-2}]\,\mathbb{E}_q[\beta_j^2]\sum_t \boldsymbol{\eta}_j^\top\mathbf{r}_t^{(j)}\mathbf{r}_t^{(j)\top}\boldsymbol{\eta}_j - \tfrac{1}{2}\sigma_\eta^{-2}\boldsymbol{\eta}_j^\top\boldsymbol{\eta}_j,
\]
yielding the precision matrix~\eqref{eq:Sigma_eta}.

For the linear term, the key calculation is $\mathbb{E}_q[\beta_j \cdot e_t^{(-j)}]$. Since $\beta_j = \mathbf{e}_{j+1}^\top\boldsymbol{\xi}$:
\begin{align*}
\mathbb{E}_q[\beta_j(y_t - \mathbf{g}_t^\top\boldsymbol{\xi})] &= \mu_{\beta_j}\,y_t - \mathbf{g}_t^\top\mathbb{E}_q[\boldsymbol{\xi}\boldsymbol{\xi}^\top]\mathbf{e}_{j+1} \\
&= \mu_{\beta_j}\,y_t - \mathbf{g}_t^\top(\boldsymbol{\mu}_\xi\boldsymbol{\mu}_\xi^\top + \boldsymbol{\Sigma}_\xi)\mathbf{e}_{j+1} \\
&= \mu_{\beta_j}\,\bar{e}_t^{(-j)} - \mathbf{g}_t^\top\boldsymbol{\Sigma}_\xi\mathbf{e}_{j+1},
\end{align*}
where the last term is the covariance correction. Collecting and applying $\boldsymbol{\mu}_{\eta_j} = \boldsymbol{\Sigma}_{\eta_j} \times (\text{linear coefficient})$ yields~\eqref{eq:mu_eta}. \hfill$\square$

\section{ELBO computation}

The ELBO decomposes as:

\textbf{Expected log-likelihood:}
\[
\mathbb{E}_q[\log p(\mathbf{y}\mid\boldsymbol{\xi},\boldsymbol{\eta},\sigma^2)] = -\tfrac{T}{2}\log(2\pi) - \tfrac{T}{2}(\log\tilde{b} - \psi(\tilde{a})) - \tfrac{\tilde{a}}{2\tilde{b}}\sum_t \mathbb{E}_q[e_t^2].
\]

\textbf{Expected log-priors:}
\begin{gather*}
\mathbb{E}_q[\log p(\boldsymbol{\xi})] = -\tfrac{J+1}{2}\log(2\pi) - \tfrac{1}{2}\log|\boldsymbol{\Lambda}_\xi^{-1}| - \tfrac{1}{2}\mathrm{tr}(\boldsymbol{\Lambda}_\xi(\boldsymbol{\mu}_\xi\boldsymbol{\mu}_\xi^\top + \boldsymbol{\Sigma}_\xi)), \\
\mathbb{E}_q[\log p(\boldsymbol{\eta}_j)] = -\tfrac{P_j-1}{2}\log(2\pi\sigma_\eta^2) - \tfrac{1}{2\sigma_\eta^2}(\|\boldsymbol{\mu}_{\eta_j}\|^2 + \mathrm{tr}(\boldsymbol{\Sigma}_{\eta_j})), \\
\mathbb{E}_q[\log p(\sigma^2)] = a_0\log b_0 - \log\Gamma(a_0) - (a_0+1)(\log\tilde{b} - \psi(\tilde{a})) - b_0\tilde{a}/\tilde{b}.
\end{gather*}

\textbf{Entropies:}
\begin{gather*}
H[q(\boldsymbol{\xi})] = \tfrac{J+1}{2}(1+\log(2\pi)) + \tfrac{1}{2}\log|\boldsymbol{\Sigma}_\xi|, \\
H[q(\boldsymbol{\eta}_j)] = \tfrac{P_j-1}{2}(1+\log(2\pi)) + \tfrac{1}{2}\log|\boldsymbol{\Sigma}_{\eta_j}|, \\
H[q(\sigma^2)] = \tilde{a} + \log(\tilde{b}\,\Gamma(\tilde{a})) - (1+\tilde{a})\psi(\tilde{a}).
\end{gather*}

\section{Additional Monte Carlo results}

\begin{table}[H]
\centering
\caption{Robustness results (Tier 2, Panel A): Weight profile shapes. $T=200$, $J=3$, $P=3$.}
\small
\begin{tabular}{@{}llccccc@{}}
\toprule
Profile & Method & Bias($\beta$) & Cov95($\beta$) & Bias($\eta$) & Cov95($\eta$) & ESS \\
\midrule
Decreasing   & CAVI  & 0.051 & 0.749 & 0.012 & 0.921 & --- \\
Decreasing   & Gibbs & 0.077 & 0.900 & 0.023 & 0.946 & 643 \\
Hump-shaped  & CAVI  & 0.057 & 0.819 & 0.010 & 0.942 & --- \\
Hump-shaped  & Gibbs & 0.079 & 0.899 & 0.019 & 0.961 & 1,034 \\
U-shaped     & CAVI  & 0.064 & 0.827 & 0.006 & 0.962 & --- \\
U-shaped     & Gibbs & 0.086 & 0.891 & 0.009 & 0.970 & 1,108 \\
\bottomrule
\end{tabular}
\end{table}

\begin{table}[H]
\centering
\caption{Robustness (Panel B): Signal-to-noise ratio and frequency ratio.}
\small
\begin{tabular}{@{}llccccc@{}}
\toprule
Variation & Method & Bias($\beta$) & Cov95($\beta$) & Bias($\eta$) & Cov95($\eta$) & ESS \\
\midrule
$\sigma_0^2=0.25$ (high SNR) & CAVI  & 0.010 & 0.899 & 0.001 & 0.948 & --- \\
                              & Gibbs & 0.012 & 0.953 & 0.002 & 0.959 & 3,280 \\
$\sigma_0^2=4$ (low SNR)     & CAVI  & 0.112 & 0.747 & 0.011 & 0.943 & --- \\
                              & Gibbs & 0.148 & 0.861 & 0.011 & 0.966 & 883 \\
$K=5$ (quarterly)             & CAVI  & 0.079 & 0.793 & 0.010 & 0.936 & --- \\
                              & Gibbs & 0.112 & 0.867 & 0.009 & 0.957 & 696 \\
$K=65$ (daily)                & CAVI  & 0.332 & 0.483 & 0.012 & 0.952 & --- \\
                              & Gibbs & 0.523 & 0.497 & 0.008 & 0.981 & 270 \\
\bottomrule
\end{tabular}
\end{table}

\begin{table}[H]
\centering
\caption{Stress test results (Tier 3).}
\small
\begin{tabular}{@{}lllcccccc@{}}
\toprule
$T$ & $J$ & Method & Bias($\beta$) & RMSE($\beta$) & Cov95($\beta$) & Bias($\eta$) & Cov95($\eta$) & ESS \\
\midrule
50  & 10 & CAVI  & 0.240 & 0.383 & 0.374 & 0.068 & 0.957 & --- \\
50  & 10 & Gibbs & 0.266 & 0.382 & 0.842 & 0.059 & 0.993 & 366 \\
200 & 10 & CAVI  & 0.156 & 0.273 & 0.441 & 0.026 & 0.955 & --- \\
200 & 10 & Gibbs & 0.203 & 0.293 & 0.708 & 0.016 & 0.981 & 339 \\
50  & 3  & CAVI  & 0.350 & 0.573 & 0.455 & 0.038 & 0.933 & --- \\
50  & 3  & Gibbs & 0.433 & 0.599 & 0.783 & 0.048 & 0.980 & 646 \\
\bottomrule
\end{tabular}
\end{table}

\section{Prior and initialization sensitivity}

\paragraph{Prior sensitivity.} Table~\ref{tab:prior_sensitivity} reports CAVI performance under three settings of the prior variance $\sigma_\beta^2 \in \{5, 10, 20\}$ for the impact coefficients, with $J=3$, $T=200$, 100 replications. Results are virtually identical across specifications, confirming robustness to the choice of prior variance over a fourfold range.

\begin{table}[H]
\centering
\caption{Prior sensitivity: CAVI performance under different $\sigma_\beta^2$. $J=3$, $T=200$, 100 replications.}
\label{tab:prior_sensitivity}
\small
\begin{tabular}{@{}cccccc@{}}
\toprule
$\sigma_\beta^2$ & Bias($\beta$) & RMSE($\beta$) & Cov95($\beta$) & Bias($\eta$) & Cov95($\eta$) \\
\midrule
5  & 0.128 \tiny{(.007)} & 0.140 \tiny{(.008)} & 0.875 \tiny{(.033)} & 0.021 & 0.932 \\
10 & 0.128 \tiny{(.007)} & 0.140 \tiny{(.008)} & 0.880 \tiny{(.032)} & 0.021 & 0.932 \\
20 & 0.127 \tiny{(.007)} & 0.140 \tiny{(.008)} & 0.880 \tiny{(.032)} & 0.021 & 0.932 \\
\bottomrule
\end{tabular}
\end{table}

\paragraph{Initialization sensitivity.} We compare the default OLS warm-start initialization against random initialization ($\mu_\xi \sim \mathcal{N}(\mathbf{0}, 0.01\mathbf{I})$, $\mu_{\eta_j} \sim \mathcal{N}(\mathbf{0}, 0.01\mathbf{I})$) over 50 replications with $J=3$, $T=200$. OLS initialization yields substantially better results: mean absolute bias of 0.121 (SE~0.011) versus 0.634 (SE~0.071) for random initialization, with the maximum ELBO difference across replications reaching 76.3. This sensitivity arises because the ELBO landscape of the bilinear model has multiple local optima; the OLS warm-start directs the algorithm toward the basin containing the global optimum. We recommend OLS initialization as the default for all applications.

\section{Computational timing}

Table~\ref{tab:timing} reports the mean CAVI and Gibbs estimation times across all Monte Carlo configurations, together with speedup factors and the minimum effective sample size (ESS) of the Gibbs chain. The speedup ranges from $107\times$ ($J=25$, 50 replications) to $1{,}772\times$ ($J=1$). For $J=50$, Gibbs timing is not reported as the computational cost was prohibitive.

\begin{table}[H]
\centering
\caption{Computational timing across Monte Carlo configurations. CAVI and Gibbs times are per-fit averages (500 replications except $J=25$ Gibbs: 50 replications).}
\label{tab:timing}
\footnotesize
\begin{tabular}{@{}lrrrrc@{}}
\toprule
Config & CAVI (s) & Iters & Gibbs (s) & Speedup & Min ESS \\
\midrule
1A-1 ($J\!=\!1$)   & 0.001 &  3 & 1.82  & 1,772$\times$ & 4,103 \\
1A-2 ($J\!=\!3$)   & 0.006 & 10 & 3.59  & 645$\times$ & 1,066 \\
1A-3 ($J\!=\!5$)   & 0.023 & 29 & 6.54  & 283$\times$ & 701 \\
1A-4 ($J\!=\!10$)  & 0.066 & 37 & 15.56 & 238$\times$ & 539 \\
1A-5 ($J\!=\!25$)  & 0.322 & 49 & 34.44 & 107$\times$ & 134 \\
1A-6 ($J\!=\!50$)  & 1.383 & 67 & ---   & --- & --- \\
\midrule
2D-3 ($K\!=\!65$)  & 0.030 & 59 & 4.06  & 136$\times$ & 270 \\
\bottomrule
\end{tabular}
\end{table}

\section{ELBO values}

Table~\ref{tab:elbo} reports the final ELBO values across all Monte Carlo configurations. The ELBO increases (becomes less negative) with higher SNR and decreases with more predictors, as expected from the additional model complexity.

\begin{table}[H]
\centering
\caption{Final ELBO values across Monte Carlo configurations (500 replications).}
\label{tab:elbo}
\small
\begin{tabular}{@{}lrrrr@{}}
\toprule
Config & Mean & Std & Min & Max \\
\midrule
1A-1 ($J\!=\!1$)    & $-$235.0 & 10.1 & $-$264.9 & $-$210.6 \\
1A-2 ($J\!=\!3$)    & $-$253.5 &  9.5 & $-$285.0 & $-$220.1 \\
1A-3 ($J\!=\!5$)    & $-$263.7 &  9.5 & $-$290.0 & $-$228.5 \\
1A-4 ($J\!=\!10$)   & $-$291.3 &  9.8 & $-$318.3 & $-$263.9 \\
1A-5 ($J\!=\!25$)   & $-$370.9 &  8.2 & $-$390.8 & $-$338.2 \\
1A-6 ($J\!=\!50$)   & $-$503.5 &  7.7 & $-$527.3 & $-$480.0 \\
\midrule
2B-1 (high SNR) & $-$100.6 &  9.8 & $-$130.4 & $-$70.1 \\
2B-2 (low SNR)  & $-$318.2 &  9.6 & $-$345.4 & $-$286.8 \\
\bottomrule
\end{tabular}
\end{table}

\section{Additional figures}

\begin{figure}[H]
\centering
\includegraphics[width=0.7\textwidth]{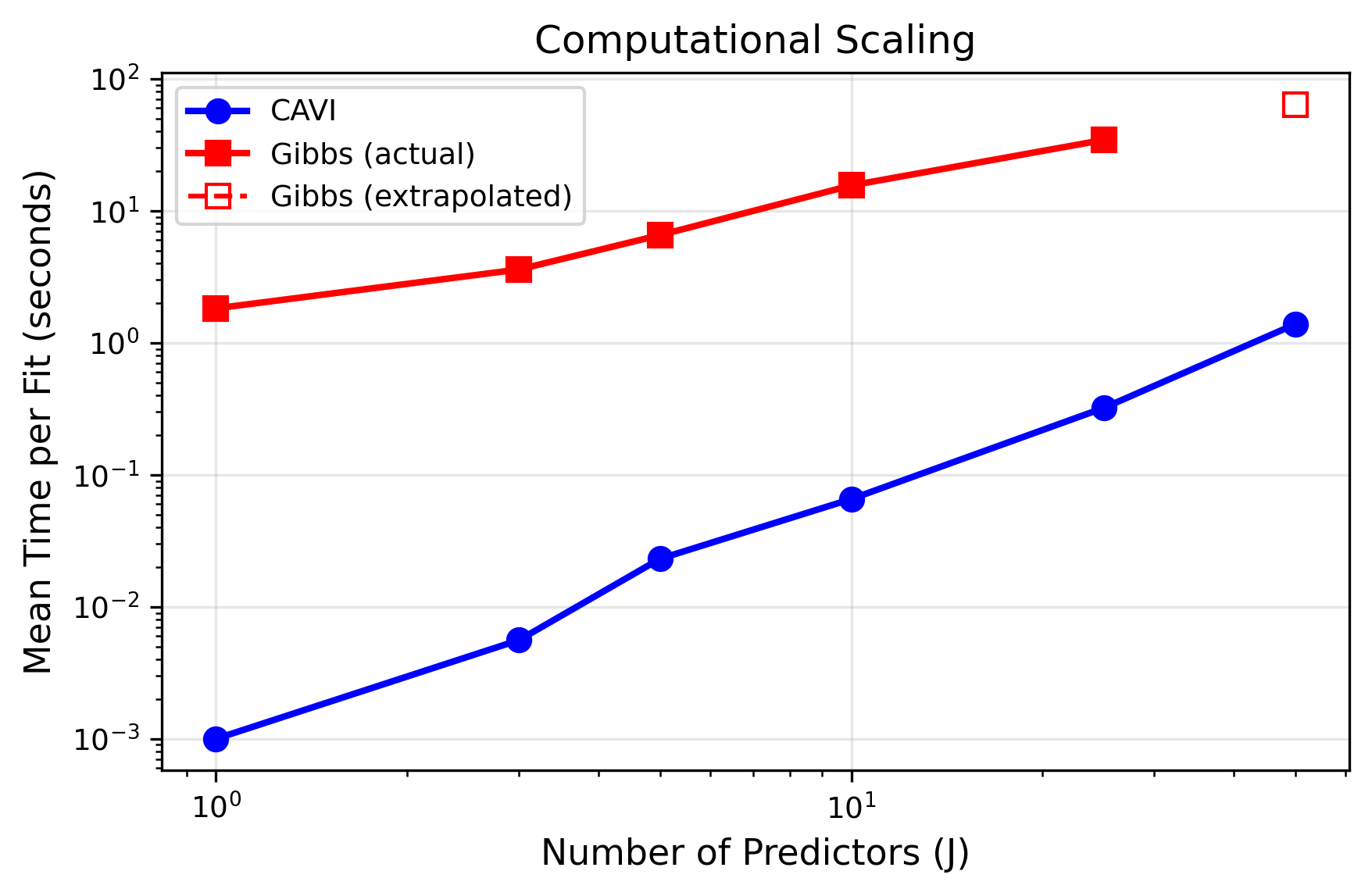}
\caption{Computational scaling: mean estimation time per fit as a function of~$J$ (log-log scale). Blue: CAVI. Red: Gibbs (solid: measured; open: extrapolated). The gap widens with $J$ due to Gibbs mixing degradation.}
\label{fig:timing}
\end{figure}

\begin{figure}[H]
\centering
\includegraphics[width=\textwidth]{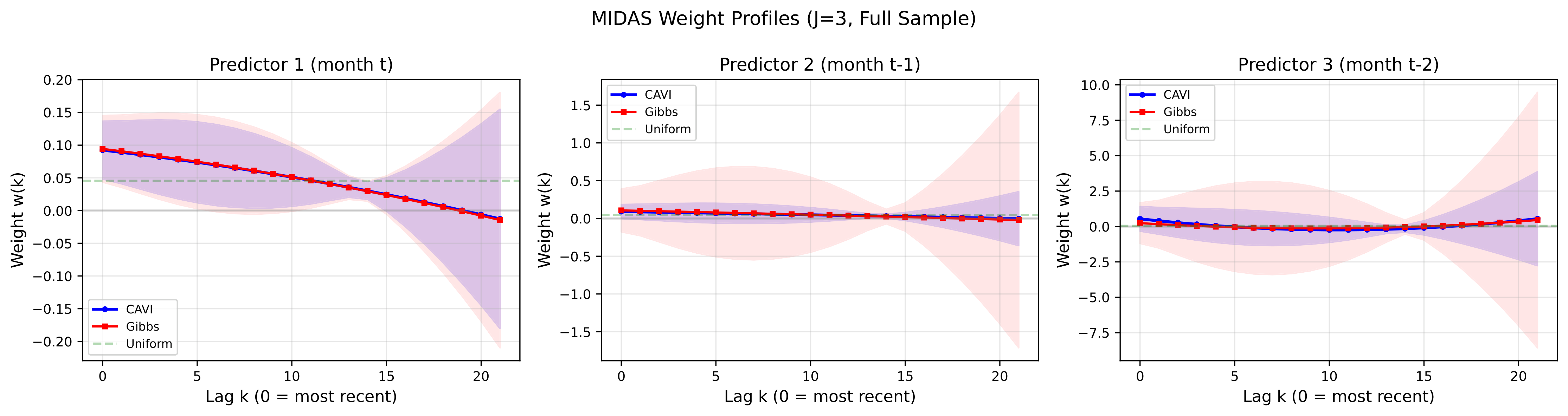}
\caption{Estimated MIDAS weight profiles ($J=3$, full sample). Left: Predictor~1 (month~$t$) shows a monotonically decreasing profile with tight credible bands, consistent with the $J=1$ result (Figure~\ref{fig:weights_emp}). Center and Right: Predictors~2 and~3 (months~$t{-}1$ and~$t{-}2$) have weights concentrated near zero with wide credible bands, indicating negligible marginal contribution beyond the current month's squared returns. CAVI (blue) and Gibbs (red) point estimates are virtually identical across all three predictors.}
\label{fig:weights_j3}
\end{figure}

\begin{figure}[H]
\centering
\includegraphics[width=\textwidth]{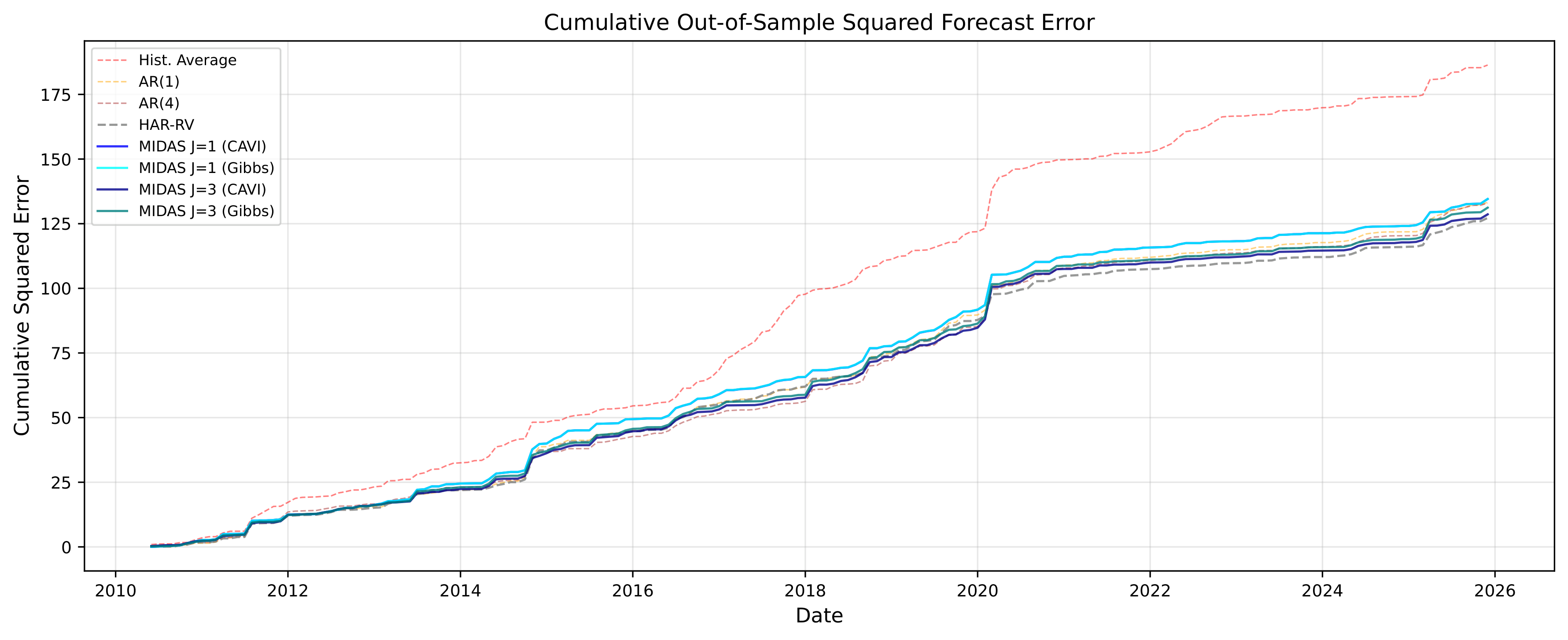}
\caption{Cumulative out-of-sample squared forecast error (2010--2025). The historical average (pink) is clearly dominated. All other models cluster closely, with a sharp increase around March 2020 (COVID-19). MIDAS $J=3$ (CAVI) tracks HAR-RV most closely.}
\label{fig:cumsse}
\end{figure}

\begin{figure}[H]
\centering
\includegraphics[width=\textwidth]{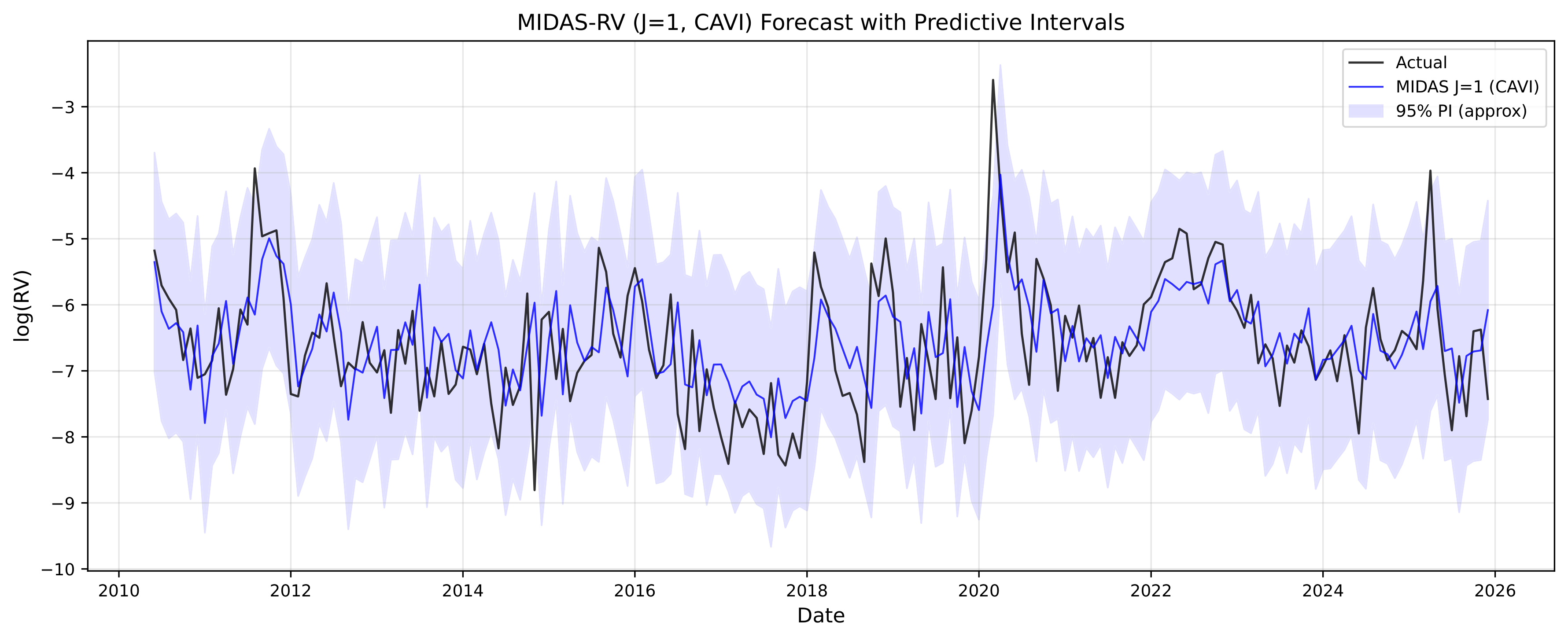}
\caption{MIDAS-RV ($J=1$, CAVI) out-of-sample forecasts with approximate 95\% predictive intervals. The March 2020 COVID-19 spike falls outside the interval, as expected for an extreme tail event.}
\label{fig:predint}
\end{figure}

\end{document}